\PassOptionsToPackage{dvipsnames}{xcolor}
\RequirePackage[l2tabu, orthodox]{nag}
\documentclass[12pt]{article}
\usepackage{mathtools} 
\usepackage{booktabs} 
\usepackage{tikz} 

\usepackage{todonotes}
\usepackage[parfill]{parskip}
\usepackage[margin=1in]{geometry}
\usepackage[defaultlines=3,all]{nowidow}

\usepackage[T1]{fontenc}
\usepackage{microtype}
\usepackage[english]{babel}
\usepackage{libertine}             
\usepackage{titlesec}
\titleformat{\section}[hang]{\Large\bfseries}{\thesection}{1em}{}
\usepackage{amsmath}
\usepackage{amssymb}
\usepackage{amsthm}
\usepackage{mathtools}

\def\RR{{\mathbb R}}    
\def\11{{\mathbf 1}}    

\def\cA{{\mathcal A}}             \def\cO{{\mathcal O}}     \def\cP{{\mathcal P}}     \def\cQ{{\mathcal Q}}  \def\cF{{\mathcal F}}    \def\cX{{\mathcal X}}   

\def\co{{\text{co}}} 
\def\ext{{\text{ext}}}

\def\RR{{\mathbb R}}    
\def\11{{\mathbf 1}}    

\def\cA{{\mathcal A}}             \def\cO{{\mathcal O}}     \def\cP{{\mathcal P}}     \def\cQ{{\mathcal Q}}  \def\cF{{\mathcal F}}    \def\cX{{\mathcal X}}

\usepackage{adjustbox}

\usepackage{comment}             
\usepackage{wrapfig}             
\usepackage{caption}             
\usepackage{pifont}              
\usepackage{minitoc}             
\usepackage{algorithm}           
\usepackage[noend]{algpseudocode}
\usepackage{nicefrac}            
\usepackage{booktabs}            
\usepackage{tikz}
\usetikzlibrary{tikzmark}        
\usepackage{multirow}
\usepackage{stfloats}
\usepackage{bm}
\usepackage[inline]{enumitem}
\usepackage{multirow}
\usepackage[normalem]{ulem}
\usepackage[toc,page,header]{appendix}
\usepackage{minitoc}
\usepackage{centernot}
\usepackage{enumitem}
\usepackage{fdsymbol}
\usepackage{hyperref}
\usepackage{cleveref}
\usepackage{subcaption}
\usepackage[dvipsnames]{xcolor}

\usepackage[sort]{natbib}
\usepackage{bibunits}

\hypersetup{colorlinks,linktoc=all}
\usepackage[all]{hypcap}
\hypersetup{citecolor=Violet}
\hypersetup{linkcolor=black}
\hypersetup{urlcolor=MidnightBlue}

\theoremstyle{plain}
\newtheorem{theorem}{Theorem}[section]
\newtheorem{proposition}[theorem]{Proposition}
\newtheorem{lemma}[theorem]{Lemma}

\theoremstyle{definition}
\newtheorem{definition}[theorem]{Definition}

\theoremstyle{remark}
\newtheorem{remark}[theorem]{Remark}

\title{Truthful Elicitation of Imprecise Forecasts}

\author{
  Anurag Singh, Siu Lun Chau, and Krikamol Muandet\\
     CISPA Helmholtz Center for Information Security, Saarbr\"ucken, 
    Germany   \\
  anurag.singh@cispa.de, siu-lun.chau@cispa.de, muandet@cispa.de\\
  }

\date{February 11, 2025}

\begin{document}
\maketitle

\begin{bibunit}[alp]

\begin{abstract}
    The quality of probabilistic forecasts is crucial for decision-making under uncertainty. While proper scoring rules incentivize truthful reporting of precise forecasts, they fall short when forecasters face epistemic uncertainty about their beliefs, limiting their use in safety-critical domains where decision-makers~(DMs) prioritize proper uncertainty management. To address this, we propose a framework for scoring \emph{imprecise forecasts}---forecasts given as a set of beliefs. Despite existing impossibility results for deterministic scoring rules, we enable truthful elicitation by drawing connection to social choice theory and introducing a two-way communication framework where DMs first share their aggregation rules (e.g., averaging or min-max) used in downstream decisions for resolving forecast ambiguity. This, in turn, helps forecasters resolve indecision during elicitation. We further show that truthful elicitation of imprecise forecasts is achievable using proper scoring rules randomized over the aggregation procedure. Our approach allows DM to elicit and integrate the forecaster's epistemic uncertainty into their decision-making process, thus improving credibility.
\end{abstract}


\doparttoc 
\faketableofcontents 

\section{Introduction}
\begin{figure*}
    \centering
    \includegraphics[width=\linewidth]{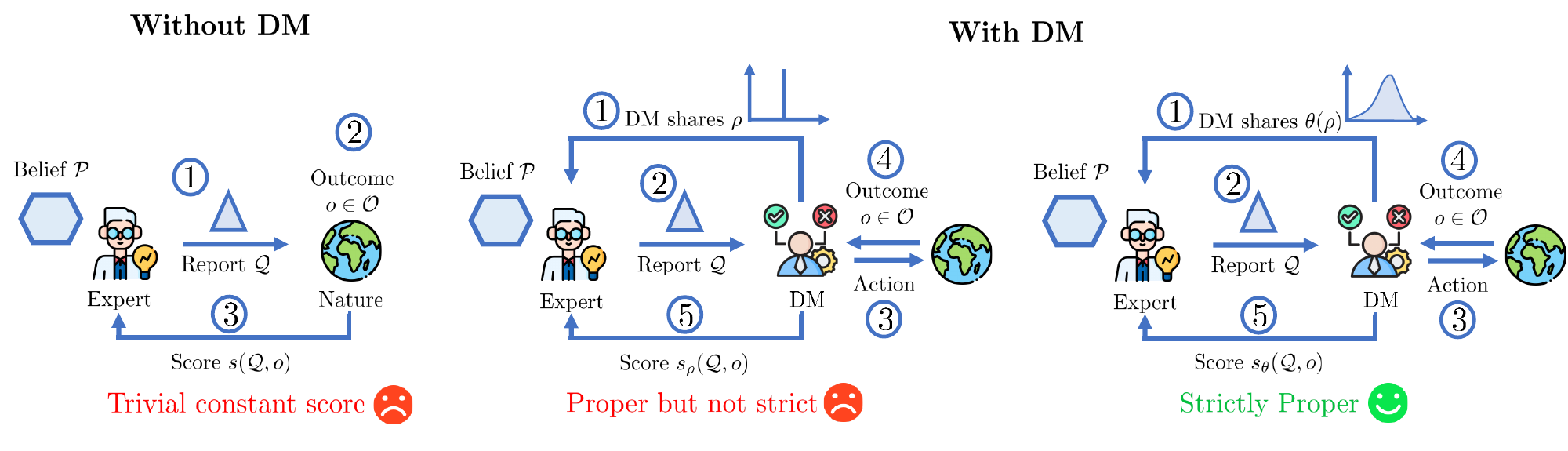}
    \caption{
    We consider scenarios where the expert holds an imprecise belief over the outcome $o \in \cO$, represented as $\cP \subseteq \Delta(\cO)$. The goal is to truthfully elicit this belief, i.e., the best report $Q$ should be $\cP$. The leftmost figure directly extends precise scoring rules to the imprecise case, ignoring the downstream DM. Truthful elicitation in the imprecise setting requires the DM to share their aggregation rule $\rho$ with the expert (middle). To avoid DM's strategic manipulation by the forecaster, DM shares a distribution $\theta(\rho)$ over aggregation rules (right), resulting in a strictly proper scoring rule $s_\theta$.}
    \label{fig:main-fig}
\end{figure*}
Probabilistic forecasting is a powerful tool for decision-making under uncertainty with diverse applications ranging from energy demand forecasting~\citep{pinson2012evaluating,pinson2013wind} and credit risk assessment~\citep{rindt2022survival,yanagisawa2023proper} to machine learning (ML)~\citep{singh2023robust} and large language models (LLMs)~\citep{shao2024language,wu2024elicitationgpt}.
Proper scoring rules serve as fundamental tools for evaluating the quality of probabilistic forecasts~\citep{brier1950verification,murphy1988decomposition,gneiting2007strictly}. They also serve as a backbone for eliciting other distributional properties such as their moments~\citep{frongillo2014general}. By assigning numerical scores based on the reported forecast and the realized outcome, these rules incentivize truthful reporting, i.e., any deviation from the forecaster’s true beliefs would result in suboptimal scores. Beyond applications in statistics, proper scoring rules have a deep connection with mechanism design, a sub-field of economics. When used as a payment mechanism, the agents have no incentive to lie, a property known as incentive compatibility~\citep{myerson1981optimal}. 

Traditionally, scoring rules operate under the assumption that forecasters possess a \emph{precise} probabilistic belief about some uncertain event. They are designed to reward the forecasters whose forecasts reflect their true precise beliefs~\citep{savage1971elicitation, gneiting_probabilistic_2014}. 
For example, in weather forecasting~\citep{brier1950verification} a forecaster who believes there is a 60\% chance of rain tomorrow should ideally report 60\% as their forecast.
However, in many real-world scenarios, forecasters face significant ambiguity due to the inherent complexity of atmospheric systems, coupled with limited data and model resolution, which introduce substantial imprecision \citep{wilks2011statistical}. It is thus plausible for forecasters to report imprecise probability assessments in these scenarios; for example, the chance of rain tomorrow may be assessed within the interval $[50\%, 70\%]$. Importantly, classical proper scoring rules built for precise forecasts cannot account for such additional uncertainty~\citep{konek2015}. 

Under the context of machine learning, imprecise forecasting is closely related to the concept of out-of-distribution (OOD) generalization~\citep{muandet_domain_2013,Zhou21:DG-Review}.
In standard supervised learning, where training and test data are assumed to be independent and identically distributed (i.i.d.), the predictive model reflects the learner's precise belief about the data generating process. However, in OOD generalization---where multiple training datasets are observed, and the test data may not be i.i.d. with the training data---\citet{singh2024domain} argue that the notion of generalization (e.g., average-case or worst-case optimization strategy) should be determined by the model's end user, also referred to as the decision-maker (DM). When direct interaction between the learner and the DM is not possible, \citet{singh2024domain} propose an \emph{imprecise learning} algorithm that trains a portfolio of predictors (forecasts) in advance, which are then provided to the DM. In contrast, for practical scenarios where the learner and DM can communicate, eliciting precise forecasts is straightforward using classical scoring rules. However, eliciting imprecise forecasts remains challenging due to the lack of suitable imprecise scoring rules. This gap motivates us to design appropriate imprecise scoring rules that are applicable beyond machine learning contexts. 

The key challenge to designing an appropriate scoring rule arises from the forecaster’s epistemic uncertainty. This challenge has led to several impossibility theorems for strictly proper imprecise scoring rules~\citep{seidenfeld2012forecasting,mayo2015accuracy,schoenfield2017accuracy}. However, these works focus solely on eliciting imprecise forecasts from the forecaster, overlooking the fact that probabilistic forecasts are typically used for downstream decision-making, making elicitation rarely the sole objective. Without input from the DM during elicitation, forecaster must rely solely on their imprecise belief, which contains inherent ambiguity. This often leads to indecision during elicitation---a key factor behind prior impossibility results. Recently, \citet{frohlich2024scoring} explored imprecise scoring rules involving DMs, but their analysis focused only on min-max (pessimistic) decision-making and lacked formal discussion of the DM's role. More broadly, indecision can be resolved through subjective choices beyond the min-max rule. However, it cannot be resolved by forecasters alone without eliminating their epistemic uncertainty. We argue that the DM must actively assist forecasters in navigating indecision by communicating their subjective preferences.

\textbf{Our contributions.} To address this challenge, we propose a novel setup for scoring imprecise forecasts where we consider a DM as an additional agent, who actively guides the forecaster in resolving indecision during elicitation (see \Cref{fig:main-fig} for different scenarios). Our contributions are summarized as follows:
\begin{itemize}
    \item We show that prior impossibility results stem from the lack of communication between DM and the forecaster. 
    \item We formalise DM-forecaster communication using aggregation rules from social choice theory~\citep{arrow2012social} and generalize tailored scoring rules
   ~\citep{johnstone2011tailored} to accommodate these aggregations. 
   \item We analyze the connection between axiomatic properties of aggregation rules from the social choice perspective and their impact on both truthful elicitation from the forecaster and the DM's decision-making process.
   
    \item By restricting to strategic communication, specifically by sharing only a distribution over aggregation rules, we propose a novel randomized tailored scoring rule that is strictly proper for imprecise forecasts. 
\end{itemize}
The rest of the paper is organized as follows. Section~\ref{sec:preliminaries} introduces proper scoring rules and imprecise probabilities. Section~\ref{section:setup} then formalizes the notion of an imprecise forecaster and outlines decision-making for the forecaster and DM. Next, Section~\ref{sec:imprecisescoringrules} explores imprecise scoring rules, first without communication and then with aggregation. Section~\ref{sec:rand-tailored-rules} presents strictly proper scoring rules for imprecise forecasts, while Section~\ref{sec:related-work} reviews prior work. Finally, Section~\ref{sec:discussion} concludes with a discussion of future directions.
\section{Preliminaries}
\label{sec:preliminaries}
This section introduces proper scoring rules, imprecise probabilities (IP), and credal sets. We begin by establishing the notation. Let $(\cO,\cF)$ be a measurable space where $\cO$ is a finite, discrete, non-empty set of possible outcomes (or states of nature) and $\cF$ is a corresponding sigma-algebra.  Let $O:\cO\rightarrow\mathbb{R}$ be a random variable associated with p.m.f. $p:\cO\rightarrow[0,1]$ on outcome $o\in\cO$. The probability simplex $\Delta(\cO)$ denotes the set of all probability distributions on $\cO$. Our framework involves two agents: a forecaster and a decision-maker~(DM), each with an associated utility function $u:\cX\times\cO\rightarrow\mathbb{R}$, where $\cX$ represents the decision space relevant to the agent's utility. 
Since we often refer to specific outcomes $o\in\mathcal{O}$, we will use $O$ and $o$ interchangeability. Thus, for some $x\in\mathcal{X}$, the agent's expected utility $\mathbb{E}_{O\sim p}[u(x,O)]$ is expressed as $\mathbb{E}_{o\sim p}[u(x,o)]$. For a set $\cP$, $co(\cP)$ corresponds to the convex hull and $\ext(\cP)$ to its extreme points.

\subsection{Precise Scoring Rules}
\label{section:proper-scoring-rule-preliminary}

Scoring rules incentivize a forecaster to truthfully report their probability assessments of an uncertain event~\citep{winkler1967quantification,brier1950verification}.  Specifically, a scoring rule $s:\Delta(\cO)\times \cO \rightarrow \mathbb{R}$ assigns a score of $s(q,o)$ to a forecaster with a forecast $q\in\Delta(\cO)$ when an outcome $o$ happens.
\begin{definition}
    A forecaster is precise if their true belief can be expressed as a probability distribution $p\in\Delta(\cO)$.
\end{definition}
Since classical proper scoring rules focus on truthful reporting and evaluation of \emph{precise} forecasts, we refer to them as precise scoring rules. To discourage a forecaster from making overly confident predictions, e.g., $q(o)=0$. We introduce $\textit{regular}$ precise scoring rule, i.e. $s(q,o)\in\mathbb{R}$ for all $o\in\cO$ and $s(q,o)=-\infty$ only if $q(o)=0$. 
\begin{definition}[Expected utility of the forecaster] Precise scoring rules implicitly assume that the forecaster is an expected utility-maximising agent. Therefore, for a forecaster with true belief $p$, the  utility of reporting forecast $q$ is
\begin{align}
\label{eq:expected-score-utility}
    u_{p}(q)=\mathbb{E}_{o\sim p}[s(q,o)].
\end{align}
\end{definition}
We now define a subclass of regular precise scoring rules, known as \emph{strictly} proper precise scoring rules that incentivize truthful reporting of the forecaster's belief.
\begin{definition}[Strictly Proper Precise Scoring Rule]
\label{def:strictlyproperscoringrule}
A scoring rule $s:\Delta(\cO)\times \cO\rightarrow \mathbb{R}\cup\{-\infty\}$ is strictly proper if the forecaster's true belief $p\in\Delta(\cO)$ uniquely maximizes their expected utility, i.e., for all $p,q\in\Delta(\cO)$ s.t. $q\neq p$,
\begin{align}
    \mathbb{E}_{o\sim p}[s(p,o)] > \mathbb{E}_{o\sim p}[s(q,o)].
\end{align}
\end{definition}
Some examples of strictly proper precise scoring rules are, logarithmic scoring rule $s(q,o)=a_o+b\text{ }\log(q(o))$ and quadratic scoring rule $s(q,o)=a_o+b(2q(o)-\mathbb{E}_{o\sim q}[q(o)])$ with $b\in\mathbb{R}_{+}$ and $a_o\in\mathbb{R}$ as arbitrary parameters. Proper precise scoring rules are closely related to convexity and can be characterized using convex functions as shown in ~\citet{mccarthy1956measures,savage1971elicitation,gneiting2007strictly}.

\begin{theorem}[\citealt{gneiting2007strictly}]
~\label{theorem:gneiting}
A regular scoring rule $s$ is (strictly) proper if and only if 
\begin{align}
    s(q,o) = G(q) - \sum_{o\in\cO} G'(q) dq(o) + G'(q)(o)
\end{align}
where $G:\Delta(\cO)\rightarrow \mathbb{R}$ is a (strictly) convex function and $G'(q)$ is a subgradient of $G$ at point $q$ and $G'(q)(o)$ is the value of gradient at outcome $o$.
\end{theorem}
An implication of Theorem~\ref{theorem:gneiting} is that with this characterisation of the scoring rule $s$, we can interpret $G$ as the corresponding maximum expected score ~\citep{frongillo2014general}.
The derivation of $G$ as the expected score is included in  Appendix \ref{remark:gneiting} for completeness. 

\subsection{IP and Credal Sets}
\label{subsection:Imprecise-probability-and-credal-sets}

Standard probability theory assigns a unique numerical value to each event, whereas \emph{imprecise probabilities} (IP) allows a range of plausible values to represent uncertainty in the presence of limited or ambiguous information. One common approach to modelling such uncertainty is via \emph{credal sets}. 
Given a subset $\cP\subseteq\Delta(\cO)$ of the plausible probability distributions, a credal set is defined as a closed and convex combination of $\cP$.
The assumption of convexity and closedness allows for rational decision-making~\citep{gajdos2004decision,troffaes2007decision} and satisfies axioms such as coherence~\citep{definetti1974theory,walley1991statistical}. While $\cP$ directly specifies the plausible beliefs about the state of nature, $\co(\cP)$ denotes the uncertainty inferred by a rational agent~\citep{walley1991statistical, augustin2014introduction}.

\section{A Joint Decision Framework for DM and Forecaster}
\label{section:setup}


In this work, we consider scenarios where an agent is tasked with selecting an input $x$ from a finite space of inputs $\cX := \{x_1,\ldots,x_n\}$. Agent's choice of input $x\in\cX$ and outcome of uncertain event $o\in\cO$ quantify the utility $u(x,o)$ obtained by the agent. In the case of a precise forecaster, $\cX:=\Delta(\cO)$ and Eq. \eqref{eq:expected-score-utility} shows how the precise score $u(x,o):=s(p,o)$ acts as a utility for the forecaster, underlining the decision-making aspect within elicitation. From the DM's perspective, $\cX:=\cA$ where $\cA:=\{a_1,\dots,a_m\}$ denotes the finite space of actions which DM can choose from. Depending upon the outcome $o\in\cO$, the DM obtains $u(x,o):=u(a,o)$ as the utility. 

\subsection{Decision-Making with Forecasts}
\label{sec:dmwithforecasts}
There exists a crucial difference between decision-making with imprecise forecasts v.s. precise forecasts. In the case of precise forecasts, the agent (forecaster or DM) has a precise belief or report $p\in\Delta(\cO)$. Using $p$ allows them to define a complete preference relation $\succeq_p$ over $\cX$ based on several well-established rationality frameworks~\citep{von2007theory,savage1972foundations}. Thereby, allowing the agent to select the corresponding best input $x^*$. This $x^*$ represents the best forecast to report in the case of a precise forecaster and the best action to take in the case of DM. However, in scenarios where the belief (or obtained report) for an agent is a set of presice beliefs $\cP\subseteq \Delta(\cO)$, the preference relation ($\succeq_{\cP}$) obtained on $\cX$ using $\cP$ is incomplete. In this case, a natural way to define $\succeq_{\cP}$ is based on the idea of dominance. 
\begin{definition}
\label{def:decision-making-with-imprecise-forecasts}
    Consider $\cP\subseteq \Delta(\cO)$, then the corresponding preference relation $\succeq_{\cP}$ over $\cX$ for a VNM rational~\citep{von2007theory} agent can be defined as follows: for all $x,x'\in\cX$,
    \begin{align*}
        x\succeq_{\cP} x'\quad \text{iff}\quad \mathbb{E}_{p}[u(x,o)]\geq \mathbb{E}_{p}[u(x',o)] \quad \forall p\in\cP.
    \end{align*}
\end{definition}
Unless $\cP$ is implicitly a precise forecast of type $\{p\in\Delta(\cO)\}$, the preference relation $\succeq_{\cP}$ is a partial order over $\cX$. The partial order $\succeq_{\cP}$ can be incomplete, since there can be a pair of inputs $x,x'\in\cX$ such that $x'\not\succeq_{\cP}x$ and $x\not\succeq_{\cP}x'$. In other words, $x$ and $x'$ are incomparable. This can result in indecision for the agent. This means that both the forecaster and the DM face indecision when they rely on $\cP$ for their respective tasks (elicitation or decision-making).

\subsection{Imprecise Forecaster}
Our work focuses on analyzing scoring rules in scenarios where the forecaster may be \emph{imprecise}. Specifically, we formalise the notion of an imprecise forecaster and their truthfulness below.
\begin{definition}\label{def:imprecise-forecast}
A forecaster is imprecise if their belief can be expressed as a set of distributions $\cP\subseteq\Delta(\cO)$. A report $\cQ\subseteq\Delta(\cO)$ is called an imprecise forecast, which implicitly includes precise forecasts $\cQ=\{q\}$ for some $q\in\Delta(\cO)$. 
\end{definition}
\Cref{def:imprecise-forecast} generalizes the precise setting as it allows the forecaster to express their (partial) ignorance by reporting both aleatoric uncertainties (as elements in the set) and epistemic uncertainties (as the set itself)~\citep{hullermeier_aleatoric_2021}. This subsumes both scenarios where the forecaster's belief is truly imprecise, e.g., the probability that it will rain tomorrow is $[0.6,0.8]$, and where their belief is calibrated with respect to multiple sources of potentially conflicting information, e.g., the estimated probability based on data from multiple weather stations. Moreover, this can also be interpreted as a ``collective'' report obtained from multiple (potentially conflicting) precise forecasters. Imprecise probability scoring rules can be defined analogously to precise scoring rules as follows.
\begin{definition}(Imprecise Probability Scoring Rule)
An imprecise probability (IP) scoring rule $s:2^{\Delta(\cO)}\times\cO\rightarrow\mathbb{R}\cup\{-\infty\}$ assigns a score of $s(\cQ,o)$ to a report $\cQ\subseteq \Delta(\cO)$ when the outcome $o\in\cO$ is realized.
\end{definition}
Analogous to precise setting, an IP scoring rule is \emph{regular} if $s(\cQ,o)\in\mathbb{R}$ for all $o\in\cO$, except if $q(o)=0$ for all $q\in\cQ$, then $s(\cQ,o)=-\infty$. To define regularity analogous to the precise setting we consider for all $q\in\cQ$, since otherwise reporting a vacuous set $\Delta(\cO)$ or other imprecise sets will have $-\infty$ as an incentive, thereby discouraging the forecaster from reporting their epistemic uncertainty. 
The score $s(\cQ,o)$ obtained by the forecaster induces a corresponding set of utilities $\bm{V}^\cP(\cQ)$ for the forecaster with an imprecise belief $\cP$, representing the expected utility of the imprecise score with respect to every distribution within their belief $\cP$.  We define this utility set as follows:
\begin{equation*}
    \bm{V}^\cP(\cQ)=\{\mathbb{E}_p[s(\cQ,o)]\}_{p\in\cP}
\end{equation*}
From the forecaster's perspective, this collection of expected utility functions $\bm{V}^\cP:2^{\Delta(\cO)}\rightarrow\mathbb{R}^{|\cP|}$, for each report $\cQ$ result in a range of plausible expected utility, i.e.,  
\begin{equation*}
    \mathrm{im}(\bm{V}^{\cP}(\cQ))=\Bigg[\inf_{p\in \cP}\mathbb{E}_{p}[s(\cQ,o)],\sup_{p\in \cP}\mathbb{E}_{p}[s(\cQ,o)]\Bigg]
\end{equation*}
where $\mathrm{im}$ is the image or the range of the forecaster's minimum and maximum expected score for forecast $\cP$ when its extreme points exist; see~\Cref{appendix:existence-of-extreme-points} for further details. 
While the equivalence of two precise distributions $p$ and $q$ is natural, i.e., $p=q$ or not. The equivalence of two imprecise beliefs is not obvious as they are sets of distributions. We now define the equivalence of two beliefs $\cP,\cP'$ in the context of elicitation as follows.

\begin{definition}(Equivalence of Imprecise Beliefs)
\label{def:equivalenceofimpreciseforecasts}
Two beliefs $\cP,\cP'\subseteq \Delta(\cO)$ are considered equivalent, denoted as $\cP\simeq \cP'$, if for all IP scoring rules $s$ and forecasts $\cQ\subseteq\Delta(\cO)$, we have $\mathrm{im}(\bm{V}^{\cP}(\cQ))=\mathrm{im}(\bm{V}^{\cP'}(\cQ))$.
\end{definition} 
Intuitively, two imprecise forecasts are equivalent if they yield the same range of plausible expected utilities for any scoring rule $s$ and reported forecast $\cQ$—that is, they induce identical decision-making. We now show that \Cref{def:equivalenceofimpreciseforecasts} reduces to the classic notion of equivalence between probability distributions when applied to precise forecasts.
\begin{proposition}
\label{prop:equivalenceofpreciseforecasts}
    For all $p,q\in \Delta(\cO)$, $\{p\}\simeq \{q\}$ iff $p=q$.  
\end{proposition}
With Proposition~\ref{prop:equivalenceofpreciseforecasts}, we establish that \Cref{def:equivalenceofimpreciseforecasts} generalises from the notion of equivalence of precise forecasts, i.e. distributions to imprecise forecasts. We can also characterize the equivalence of two imprecise forecasts as the equivalence of their corresponding credal sets. 
\begin{proposition}
\label{prop:credalsets}
For imprecise beliefs $\cP,\cP'\subseteq \Delta(\cO)$ with non-empty extreme points, $\cP\simeq \cP'$ iff $co(\cP)=co(\cP')$.   
\end{proposition} 
It has previously been shown that two sets of distributions must be credal sets to induce the same rational decision-making behaviour~\citep{troffaes2007decision,huntley2014decision,troffaes2014lower}. \Cref{def:equivalenceofimpreciseforecasts} defines the equivalence of two imprecise beliefs w.r.t elicitation and Proposition~\ref{prop:credalsets} establishes its equivalence to rational decision making. This allows us to consider elicitation as a decision-making task for the forecaster. 
As a consequence of Proposition ~\ref{prop:credalsets}, even though a forecaster believes in a set of probability distributions $\cP$. We restrict our focus to evaluating a credal set of forecasts $co(\cP)$. Therefore, from now on, we will assume that $\mathcal{P}$ is a convex set.


\begin{definition}(Truthfulness of Imprecise Forecaster)
\label{def:truthfulness}
Let $\cP\subseteq\Delta(\cO)$ be the true belief of an imprecise forecaster. A report $\cQ\subseteq\Delta(\cO)$ is truthful
if $\cQ\simeq\cP$.
\end{definition}

\Cref{def:truthfulness} generalizes the concept of truthfulness in the precise setting. An imprecise forecaster who reports their true belief is considered truthful. For instance, if the forecaster believes the probability of rain tomorrow lies within the interval $[0.6,0.8]$, then they must report their actual epistemic uncertainty by reporting the interval $[0.6,0.8]$.
\section{Proper IP Scoring Rules}
\label{sec:imprecisescoringrules}

In this section, we introduce proper imprecise scoring rules, i.e., scores that incentivize the truthful reporting of an imprecise forecaster according to \Cref{def:truthfulness}. We start by focusing only on the elicitation of the imprecise forecaster without any communication from the DM. The following definition clarifies what it means for an imprecise scoring rule to be (strictly) proper, which naturally generalises \Cref{def:decision-making-with-imprecise-forecasts} from the forecaster's perspective. 
\begin{definition}
\label{def:propernessofimprecisescoringrule}
An imprecise scoring rule is (strictly) proper if for all credal sets $\cP,\cQ\subseteq\Delta(\cO)$, the forecaster with an imprecise belief $\cP\not\simeq\cQ$ (strictly) prefers $\cP$ over $\cQ$, i.e., $\cP\succeq_{\cP}\cQ$. The preference relation $\succeq_{\cP}$ is described by the (strict) dominance of $\bm{V}^{\cP}(\cP)$ over $\bm{V}^{\cP}(\cQ)$, i.e.,
\begin{align*}
     \mathbb{E}_p[s(\cP,o)]\geq \mathbb{E}_p[s(\cQ,o)] \quad \text{for all }p\in\cP,
\end{align*}
for strict dominance, at least one $\geq$ has to be strictly greater.
\end{definition}
We define strict properness of an imprecise scoring rule in \Cref{def:propernessofimprecisescoringrule} using dominance since it preserves the main idea behind strictly proper scoring rules in the precise setting, i.e. to incentivise the forecaster to be truthful. A strictly proper IP scoring rule incentivises the imprecise forecaster to be truthful according to \Cref{def:truthfulness}. 
\begin{theorem}
    \label{theorem:nosaneproperscoringrule}
    There does not exist a strictly proper imprecise scoring rule $s$. In addition, for a scoring rule $s$ to be proper it must be constant across all forecasts.
\end{theorem}
Similar impossibility results for imprecise forecasts have previously been reported in \citet{seidenfeld2012forecasting,mayo2015accuracy,schoenfield2017accuracy}. The implication of \Cref{theorem:nosaneproperscoringrule} is that under the current setup of an imprecise forecaster, any imprecise scoring rule satisfying properness in \Cref{def:propernessofimprecisescoringrule} has a constant score across all forecasts. 
We observe in \Cref{sec:dmwithforecasts} that agents face possible indecision while making decision with the credal set $\cP$. As a result, we observe in \Cref{theorem:nosaneproperscoringrule} that it is not possible to design a scoring rule that incentivises the imprecise forecaster to report their imprecise belief $\cP$ honestly. Unlike in the precise setting, where the forecaster had a complete preference relation on plausible reports (see \Cref{def:strictlyproperscoringrule}), the epistemic uncertainty of the imprecise forecaster only allows for an incomplete preference relation $\succeq_{\cP}$ over plausible reports. 
Without further information, the imprecise forecaster cannot complete this incomplete preference relation. 

\subsection{Aggregation Functions}
\label{subsection:aggregation}
To resolve indecision arising from epistemic uncertainty in the credal set $\cP$, the DM exercises a subjective choice through aggregation function $\rho$ to make $\succeq_{\cP}$ complete. The DM communicates the choice of $\rho$ to the forecaster \emph{prior} to elicitation, and the elicited credal set then informs downstream decisions for the DM. The resulting utility can be shared with the forecaster as an incentive.

\begin{definition}[Aggregation Function]\label{def:aggregate}
     For a credal set $\cQ$ an aggregation function \(\rho:{(\mathbb{R}^{\cX})}^{|\cQ|}\rightarrow\mathbb{R}^{\cX}\) combines multiple utilities via a positive linear combination, i.e., for any $x\in\cX$:
    $$\rho[\{\mathbb{E}_{q}[u(x,o)]\}_{q\in\cQ}] = \int_{q\in\cQ}\bm{w}(q) \mathbb{E}_q[u(x,o)]dq$$ 
    where $\bm{w}(q) \in \RR^{|\cQ|}_{\geq 0}$ for all $q\in\cQ$ depends on the expected utilities $\{\mathbb{E}_{q}[u(x,o)]\}_{q\in\cQ}$.
\end{definition}
We focus on linear aggregations because, according to \citet{harsanyi1955cardinal}, this class of aggregation rules uniquely satisfies both VNM axioms \citep{von2007theory} and Bayes Optimality \citep{brown1981complete}. Many popular aggregation functions such as utilitarian and egalitarian rules can be expressed with linear aggregations as they characterise relative utilitarianism \citep{dhillon1999relative}. 

For the utilitarian and egalitarian rules, the decision-making process from an agent's perspective is picking an $x\in\cX$. Either for the DM an $a\in\cA$ or for the forecaster a $\cQ\subseteq\Delta(\cO)$. Illustrating this from the DM's perspective, the utilitarian rule corresponds to the linear combination $\rho[\{\mathbb{E}_q[u(a,o)]\}_{q\in\cQ}] = \nicefrac{1}{|\cQ|}\sum_{q\in \cQ}\mathbb{E}_q[u(a,o)]$, whereas the egalitarian rule corresponds to
$\rho[\{\mathbb{E}_q[u(a,o)]\}_{q\in\cQ}] = \min_{q\in \cQ}\mathbb{E}_q[u(a,o)]$. Here, the weights $\bm{w}$ can be interpreted as $\bm{w}(q)=\nicefrac{1}{|\cQ|}$ and a one-hot vector, respectively. A VNM-rational DM (see Eq. \eqref{eq:forecaster-aggregation} for forecaster) uses $\rho$ to obtain the best action  $a^*_{\cQ,\rho}$: 
\begin{equation}
\label{eq:decision-maker-opt}
    a^*_{\cQ,\rho} = \arg\max_{a\in\cA}\; \rho[\{\mathbb{E}_{q}[u(a,o)]\}_{q\in\cQ}].
\end{equation}
Given the incomplete preference relation from a credal set $\cQ$, i.e., $\succeq_{\cQ}:=\{\succeq_q\}_{q\in\cQ}$, the aggregation rule $\rho$ allows us to define the corresponding complete preference relation $\succeq_{\rho[\cQ]}$, representing the aggregated utility from Equation~\eqref{eq:decision-maker-opt}. By abuse of notation, $\succeq_{\rho[\cQ]}$ represents the aggregation of utilities rather than the credal set.

\textbf{Axiomatisation of $\rho$:} When interpreting imprecise forecasts as a “collective” report of precise forecasters, a social choice perspective naturally emerges for the downstream DM. Although non-intuitive, this perspective applies even to a single-agent imprecise forecaster. Following \citet{arrow1950difficulty}, we outline three desirable properties of any aggregation rule $\rho$: Pareto Efficiency (PE), Independence from Irrelevant Alternatives (IIA), and Non-Dictatorship (ND).  
\begin{definition}[Pareto Efficiency]
\label{def:paretoefficiency}
    An aggregation rule $\rho$ is Pareto Efficient iff for all $x,x'\in\cX$,
    \begin{equation*}
       x\succeq_{\cQ} x' \implies x\succeq_{\rho[\cQ]} x'.
    \end{equation*}
\end{definition}
From the DM's perspective, $\cX=\cA$, and as a result, a Pareto efficient $\rho$ will respect the inherent partial order $\succeq_{\cQ}$ over actions which DM could infer from the reported credal set $\cQ$. Therefore, the DM can be assured that application of $\rho$ only resolves indecision and similarly for the forecaster when choosing the best report. Additionally, from the forecaster's perspective, a non-PE $\rho$ can distort recommendations of their forecast $\cQ$ of an action $a$ over $a'$. 
The aggregation rule that violates PE may result in the payment/score that misaligns with the forecaster's report.

\begin{definition}[IIA]
\label{def:IIA-scoringrule}
An aggregation function $\rho$ is considered IIA if preferences between $x,y\in \cX$, i.e., $x\succeq_{\rho[\cQ]}y$ or $y\succeq_{\rho[\cQ]}x$ is independent of whether any other $z$ is in $\cX$.
\end{definition}
Although cryptic, IIA is desirable to the DM. From the DM's perspective, $\cX=\cA$. Imagine a scenario where there exists a $z\in \cA$ such that both $x,y\in\cA$ dominate $z$ w.r.t. the partial order $\succeq_{\cQ}$, implying that $z$ is irrelevant to the DM under forecasts $\cQ$. However, if $\rho$ violates IIA, the post-aggregation preference $\succeq_{\rho[\cQ]}$ between $x$ and $y$ can be influenced by the presence or absence of $z$. This makes the DM vulnerable to strategic manipulation regarding the best action to take by adding or removing $z$, which in turn creates uncertainty for the forecaster about their own incentives.
\begin{definition}[Non-Dictatorship]
\label{def:non-dictatorship}
    An aggregation rule $\rho$ is said to be non-dictatorial if for a profile of preferences $\succeq_{\cQ}:=\{\succeq_q\}_{q\in\cQ}$ there does not exist $q\in\cQ$ (dictator) such that for all $x,y\in \cX$, $x\succeq_{q} y$ implies $x\succeq_{\rho[\cQ]}y$.
\end{definition}
From the downstream decision-making perspective for a DM, non-dictatorship is optional. However, when DM wants to communicate the aggregation rule $\rho$ to the forecaster and wishes to truthfully elicit their true belief, non-dictatorship becomes crucial. Given a dictatorial $\rho$, the forecaster can manipulate the DM by strategically reporting the dictator $q$. We discuss this more formally in \Cref{appendix:dictatorships}.   
\subsection{Proper IP scores with aggregation}
The DM communicates the aggregation function $\rho$ to the forecaster and incentivises them using an IP scoring rule. This communication helps resolve the forecaster's epistemic uncertainty, parameterizing the IP scoring rule as $s_\rho:2^{\Delta(\cO)}\times\cO\rightarrow\mathbb{R}$. The forecaster reports $\cQ\in 2^{\Delta(\cO)}$ and receives a score of $s_\rho(\cQ,o)$ when outcome $o\in\cO$ occurs. Unlike prior IP scoring rules, the forecaster can now use $\rho$ to resolve indecision and complete the preference relation over $2^{\Delta(\cO)}$. This is evident from the expected utility of the forecaster with belief $\cP$ when reporting $\cQ\subseteq\Delta(\cO)$:
\begin{equation}
\label{eq:forecaster-aggregation}
    V^{\cP}_{\rho}(\cQ):=\rho[\bm{V}^{\cP}(\cQ)]=\rho[\{\mathbb{E}_{p}[s_\rho(\cQ,o)]\}_{p\in\cP}].
\end{equation}
Since an imprecise decision scoring rule $s_\rho$ is simply a parameterised IP scoring rule, its regularity is defined exactly as in \Cref{sec:imprecisescoringrules}. We extend (strict) properness for IP scoring rules from \Cref{def:strictlyproperscoringrule} to aggregation as follows.
\begin{definition}
    \label{def:properscoringimprecisedecisionrule}
    A regular IP scoring rule $s_\rho$ for an aggregation function $\rho$ is proper if, for all $\cP\subseteq\Delta(\cO)$ and all $\cQ\not\simeq\cP$, $V^{\cP}_\rho(\cP) \geq V^{\cP}_\rho(\cQ)$. The IP scoring rule $s_\rho$ is strictly proper if and only if at least one of the inequalities is strict.
\end{definition}
Notably, strictness in \Cref{def:properscoringimprecisedecisionrule} adheres to the notion of truthfulness defined in \Cref{def:truthfulness}. Since DM needs to evaluate the forecaster, we employ the class of scoring rules that accommodate a DM in evaluating a forecast, called tailored scoring rules \citep{dawid2007geometry,richmond2008scoring,johnstone2011tailored}, following the ideas of business sharing proposed in \citet{savage1971elicitation}. We now define them in the context of aggregation functions for imprecise forecasts.
\begin{definition}[Tailored Scoring Rules] 
\label{def:tailoredscoringrule}
An IP scoring rule $s$ is tailored for a DM with utility function $u$ and aggregation function $\rho$, if for any $k,c\in\mathbb{R}_{\geq 0}$, the score is defined as 
    \begin{equation*}
        s_{\rho}(\cQ,o)=k u(a^*_{\cQ,\rho},o) + c .
    \end{equation*}
\end{definition}
In \Cref{def:tailoredscoringrule}, $k$ can be referred to as the business share obtained by the forecaster in the utility of the DM and $c$ is the fixed fee of the forecaster. Next, we show that the class of tailored scoring rules is proper for any choice of $\rho$,
\begin{proposition}
\label{prop:tailoredscoringrule}
    A tailored scoring rule $s_\rho$ is proper with respect to \Cref{def:properscoringimprecisedecisionrule} for any aggregation rule $\rho$.
\end{proposition}
While necessary, the properness of scoring rules is easy to satisfy (see \Cref{theorem:nosaneproperscoringrule}). For example, a constant scoring rule is always  proper. We therefore characterise strict properness of $s_\rho$ for imprecise forecasts.
\begin{lemma}
Let $s_\rho$ be a tailored scoring rule. Then, the following holds:
\begin{enumerate}
\label{lemma:strictness-for-precise-distributions}
    \item $s_\rho$ is strictly proper for \textbf{precise distributions} if and only if $a^*_{q} := \arg\max_{a\in\cA}\mathbb{E}_{q}[u(a,o)]$ is a unique maximiser for all $q\in\Delta(\cO)$.
    \label{theorem:aggregation-strictness-impossibility}
    \item $s_\rho$ is not strictly proper, i.e., does not satisfy \Cref{def:properscoringimprecisedecisionrule}, for any Pareto efficient $\rho$.
\end{enumerate}
\end{lemma}
\Cref{lemma:strictness-for-precise-distributions} ensures the existence of non-constant proper IP scoring rules. Beyond this positive result, we observe that Pareto efficiency leads to the impossibility of truthful elicitation under \Cref{def:truthfulness}. Although $s_\rho$ in \Cref{lemma:strictness-for-precise-distributions} is not strictly proper for imprecise forecasts, it remains practical to implement while being proper for all forecasts and strictly proper for precise ones. We speculate that the properties of $s_\rho$ are optimal for deterministic scoring rules, given the prior impossibility of any real-valued strictly proper IP scoring rules \citep{seidenfeld2012forecasting}. To explore this further, we investigate whether allowing randomisation in the choice of aggregation rule can enable truthful elicitation.


\section{strictly proper IP scores}
\label{sec:rand-tailored-rules}

With the randomized choice of aggregation, the DM can pick an aggregation rule randomly post-elicitation to evaluate the reported forecast. The forecaster then becomes unaware of the aggregation function which can lead the forecaster back to indecision. To resolve the forecaster's indecision, the DM shares a distribution $\theta\in\Delta(\bm{\rho})$ where $\bm{\rho}$ is the class of aggregation functions the DM will pick from, thereby enabling the forecaster to resolve their indecision as follows:
\begin{equation}
\label{eq:random-tailored-scoring-truthfulness}
V^{\cP}_{\theta}(\cQ):= \mathbb{E}_{\rho\sim \theta}[V^{\cP}_{\rho}(\cQ)].
\end{equation}
This allows the tailored scoring rule $s_\rho$ to be randomized with respect to the random variable $\rho$. Analogous to \Cref{def:tailoredscoringrule} for tailored scoring rules, we now define randomized tailored scoring rule $s_{\theta}$. 
\begin{definition}
    A regular IP scoring rule $s_{\theta}$ is randomized tailored for a DM with a class of aggregation functions $\bm{\rho}$ and a distribution $\theta\in\Delta(\bm{\rho})$, if for any $k_\rho,c_\rho\in\mathbb{R}_{\geq0}$ and an arbitrary function $\Pi:2^{\Delta(\cO)}\rightarrow \mathbb{R}$, the score is defined as
    \begin{equation*}
        s_{\theta}(\cQ,o)(\rho) = \begin{cases}
        k_\rho u(a^*_{\rho,\cQ},o)+c_\rho &\text{if } \theta(\rho)>0 \\
        \Pi_o(\cQ) & \text{if } \theta(\rho)=0
        \end{cases}.
    \end{equation*}
\end{definition}
Given that we have now extended the tailored scoring rule to random variables, in a similar spirit to \Cref{def:properscoringimprecisedecisionrule} on properness of IP scoring rules with aggregation, we define properness of randomized tailored scoring rules as follows.
\begin{definition}
    \label{def:properscoring-tailored}
    A randomized tailored scoring rule $s_{\theta}$ for a distribution $\theta\in\Delta(\bm{\rho})$ and a class of aggregation rules $\bm{\rho}$, is considered proper if, for all $\cP,\cQ\subseteq\Delta(\cO)$ and $\cQ\not\simeq\cP$,
    \begin{equation}
    \label{eq:properscoring-tailored-random}
        V_\theta^\cP(\cP)\geq V_\theta^\cP(\cQ).
    \end{equation}
    $s_{\theta}$ is strictly proper if the inequality in \Cref{eq:properscoring-tailored-random} is strict.
\end{definition}
Again the strictness in \Cref{def:properscoring-tailored} adheres to the notion of truthfulness defined in \Cref{def:truthfulness}. We will establish this connection later in this section. We can observe from Equation~\ref{eq:random-tailored-scoring-truthfulness} that randomized tailored scoring rules are proper for any choice of $\theta\in\Delta(\bm{\rho})$ as a direct consequence of \Cref{prop:tailoredscoringrule}. Before we discuss how to build strictly proper IP scoring rules, we need to identify if there exists a unique representation of the credal set in the action space which will let the DM identify the credal set. 
\begin{lemma}
\label{lemma:unqiue-representation-of-credal-set}
For any reported credal set $\cQ\subseteq\Delta(\cO)$ and a DM using a utility function $u$ such that $a^*_q := \arg\max_{a\in\cA}\mathbb{E}_q[u(a,o)]$ is unique for all $q\in\Delta(\cO)$, the set of actions $\cA^{\ext}_\cQ:=\Big\{a^{*}_{q}\Big\}_{q\in\ext(\cQ)}$ acts as a unique representation of a credal set $\cQ$ in action space $\cA$.
\end{lemma}
The implication of unique representation $\cA^{\ext}_\cQ$ in the action space for any credal set $\cQ$ is that the DM is able to identify the credal set from the set of actions $\cA^{\ext}_\cQ$. In a naive analogy, all actions in $\cA^{\ext}_\cQ$ together act as a fingerprint of credal set $\cQ$ which can be uniquely incentivised by the DM to elicit $\cQ$. We now introduce a common class of linear aggregations to operationalise scoring rules based on \Cref{lemma:unqiue-representation-of-credal-set}.

\textbf{Fixed Linear Aggregations} is another common class of aggregation rules which aggregates the expected utilities of a credal set $\cQ$ for any input $x\in\cX$, i.e., $\{\mathbb{E}_q[u(x,o)]\}_{q\in\cQ}$,  into a convex combination of utilities with mixing coefficient $\bm{\lambda}\in\Delta^{|\cQ|}$ as 
\begin{align*}
    \rho_{\bm{\lambda}}[\{\mathbb{E}_q[u(x,o)]\}_{q\in\cQ}] &:= \int_{q\in\cQ}\bm{\lambda}(q)\mathbb{E}_q[u(x,o)]dq\\
    &=\mathbb{E}_{\int\bm{\lambda}(q)q dq}[u(x,o)] .
\end{align*}
Although the class of fixed linear aggregations is Pareto-efficient and non-dictatorial in classic social choice theory, in our setup, fixed linear aggregations are dictatorships as they directly aggregate the epistemic uncertainty. Due to \Cref{prop:credalsets}, a forecaster can report $\cQ$ or $\co(\cQ)$. We illustrate this with an example, for any report $\cQ\subseteq\Delta(\cO)$ and any choice of fixed linear aggregation $\bm{\lambda}$, we obtain $Q:=\bm{\lambda}^\top \cQ$. Even though $Q$ may not be in $\cQ$, it is guaranteed that $Q\in \co(\cQ)$, and therefore $Q$ acts as a dictator. This means that although the DM uses the full credal set in the sense of all extreme points to perform decision-making, their preference over actions can be fully represented by a precise belief $P\in co(\cP)$. 
From \Cref{subsection:aggregation}, non-dictatorship was only desirable due to the strategic manipulation by the forecaster. In the scenario where forecasters are unaware of the exact aggregation rule, using a random dictatorial $\rho_{\bm{\lambda}}$ allows the DM to keep PE and IIA. To this end, we show the strict properness of these randomized dictatorships. Since strict properness for imprecise forecasters implicitly requires strictness for precise forecasts, which means that the $s_{\theta}$ must satisfy \Cref{lemma:strictness-for-precise-distributions} for every $\bm{\lambda}$.
\begin{theorem}
\label{theorem:strictly-proper-imprecise-scoring-rule}
    Assuming $s_\theta$ to be strictly proper for precise distributions and $\bm{\rho}$ as fixed linear aggregations, $s_{\theta}$ is strictly proper for imprecise forecasts, i.e. $s_\theta$ is a strictly proper IP scoring rule if $\theta$ has full support over $\bm{\rho}$.
\end{theorem}
\Cref{theorem:strictly-proper-imprecise-scoring-rule} allows us to build strictly proper IP scoring rules which can be characterized as follows. A randomized tailored scoring rule $s_{\theta}$ made using the class of fixed linear aggregation rules is characterized as
\begin{align*}
    s_{\theta}(\cQ,o)(\bm{\lambda}) = \begin{cases}
        k_{\bm{\lambda}} u(a^*_{\rho_\lambda,\cQ},o)+c_{\bm{\lambda}} &\text{if } \theta(\bm{\lambda})>0 \\
        \Pi_o(\cQ) & \text{if } \theta(\bm{\lambda})=0
        \end{cases},
\end{align*}
where $\bm{\lambda}\in\Delta^{|\ext(\cQ)|}$ is considered strictly proper if $\text{supp}(\theta)=[0,1]$ where $\Pi:2^{\Delta(\cO)}\times \cO\rightarrow \mathbb{R}$ is an arbitrary regular scoring function. To verify the strict properness of our score, we conduct a simulation (see \Cref{appendix:simulation}).

In recent years, several frameworks have been proposed for learning that challenge the implicit assumptions made in standard ML pipeline about loss functions \citep{gopalan2021omnipredictors}) or preferences \citep{singh2024domain} of the users being known to the model trainer. They focus on training models that perform well for a class of losses or aggregation rules. Within our setup, these frameworks translate to the DM abstaining from sharing the exact aggregation rule with the forecaster. However, they are not exact implementations of the score we propose. Applying the proposed score to ML applications is one of the future research avenues. 

\vspace{-0.5em}
\section{Related Work}
\vspace{-0.5em}
\label{sec:related-work}
The work of \citet{frohlich2024scoring} is most closely related to ours. They also explore the generalization of proper scoring rules to imprecise forecasts, with a specific emphasis on calibration~\citep{dawid1982well}. While their focus is on imprecisions arising from data models, we address more general issues related to the elicitation of imprecise forecasts. Their findings demonstrate that, unlike in precise settings where proper scoring and calibration objectives align, these goals can diverge when dealing with imprecise forecasts—a result that parallels our own. However, their reliance on the min-max aggregation within their scoring framework limits their analysis to pessimistic decision-making, resulting in a scoring rule that only satisfies properness.


Impossibility results show that no continuous scoring rule over credal sets can satisfy strict incentive compatibility, calibration, and non-domination simultaneously~\citep{seidenfeld2012forecasting, mayo2015accuracy, schoenfield2017accuracy}. \citet{seidenfeld2012forecasting} proved that such rules must either weaken incentive compatibility or permit domination by precise forecasts. \citet{mayo2015accuracy} highlighted that these trade-offs can inadvertently reward false precision, while \citet{schoenfield2017accuracy} showed that any continuous rule is either constant or fails to calibrate in natural decision contexts. While our approach partly mitigates these issues, these impossibility results still constrain deterministic methods. Some view the lack of imprecise scoring rules analogous to precise ones as a fundamental trait of imprecision \citep{konek2015}. Building on this, \citet{konek2019} proposes a family of IP scoring rules based on the Hurwicz criterion, extended by \citet{konek2023} to formalize precision–robustness trade-offs axiomatically. Since the Hurwicz criterion yields Pareto-efficient aggregation, our results in Section~\ref{sec:imprecisescoringrules} directly apply to their framework, offering a social choice lens on these trade-offs.

Finally, our work is uniquely positioned at the intersection of proper scoring rules, forecast elicitation, and machine learning, providing novel perspectives on decision-making under uncertainty. Credal sets have become a mainstream approach for representing modelers’ imprecision with applications in prediction~\citep{singh2024domain, caprio_credal_2024}, uncertainty quantification~\citep{sale_is_2023,wang2024credal}, optimal transport~\citep{caprio2024optimal}, statistical hypothesis testing~\citep{chau2025credal}, and statistical distances~\citep{chau2025integral}, among others. To this end, our results concerning strictly proper scoring rules for credal sets are directly relevant to the challenges of learning and decision-making with credal sets, providing insights into fundamental problems and future research directions.

\vspace{-0.5em}
\section{Discussion}
\vspace{-0.5em}
\label{sec:discussion}
Our investigation of strictly proper IP scoring rules reveals that, unlike in the classical precise setting, forecasting under imprecision demands careful attention to the decision-making aspect within forecast elicitation. In traditional frameworks with strictly proper scoring rules, forecasters are simply expected utility maximizers, making the reporting decision straightforward. However, when forecasts are imprecise—represented as sets or intervals—forecasters cannot internally aggregate their epistemic uncertainty. Instead, they require an external aggregation rule to reconcile their credal set-induced preferences into a single forecast. 

This need for external decision guidance naturally connects to social choice theory. In our framework, the DM provides a collective aggregation rule that guides forecasters in resolving their uncertainty. This approach not only preserves incentive compatibility in the imprecise setting but also highlights the importance of designing scoring rules that balance accuracy and robustness. By explicitly integrating a social choice–inspired aggregation function into the elicitation process, our work offers new perspectives on collective decision-making, where imprecise forecasts can be viewed as forecasts of the ``collective.'' This highlights promising directions for future research on imprecise scoring rules. 



\newpage
\addcontentsline{toc}{section}{Appendix} 
\part{Appendix} 
\parttoc 

\appendix
\section{Additional Supporting lemmas and proofs}
\subsection{Proof of Remark~\ref{remark:gneiting}}
\begin{remark}
~\label{remark:gneiting}
    Scoring rule $s$ is (strictly) proper if and only if the corresponding (strictly) convex function $G(q)= \sum_{o\in\cO} s(q,o)q(o)$ 
\end{remark}
\begin{proof}
It follows from Theorem ~\ref{theorem:gneiting} that regular scoring rule $s$ is (strictly) proper if and only if there exists a corresponding (strictly) convex function $G$ on $\Delta(\cO)$ such that 
\begin{align}
    s(q,o)= G(q) - \sum_{o\in\cO} G'(q)(o) q(o) + G'(q)(o).
\end{align}
$(\Rightarrow)$ 
Let us assume that there exists a strictly proper scoring rule $s$. Then, according to Theorem ~\ref{theorem:gneiting} there exists a convex function $G:\Delta(\cO)\rightarrow \mathbb{R}$ such that 
\begin{align*}
    s(q,o) &= G(q) -  \sum_{o\in\cO} G'(q)(o) dq(o) + G'(q)(o) \\
    \mathbb{E}_{o\sim p}[s(q,o)] &= \mathbb{E}_{o\sim p}\bigg[G(q) -  \sum_{o\in\cO} G'(q)(o) dp(o) + G'(q)(o)\bigg] \tag*{(expert's true belief \mbox{$p$})} \\
    &= G(q) -  \sum_{o\in\cO} G'(q)(o) q(o) +  \sum_{o\in\cO} G'(q)(o) p(o),
\end{align*}
where $q$ is the true belief of the forecaster. Then, we consider the maximum expected score
\begin{align*}
     \sum_{o\in\cO} s(q,o)q(o) &:= \max_{p\in \Delta(\cO)} u_q(p)\tag*{(\mbox{$s$} is strictly proper)}\\
     &= \max_{p\in \Delta(\cO)} \mathbb{E}_{o\sim q}[s(p,o)]\\
    &= \max_{p\in \Delta(\cO)} \left\{G(p) -  \sum_{o\in\cO} G'(p)(o) p(o) +  \sum_{o\in\cO} G'(p)(o) q(o) \right\} \\
    &=G(p^*) -  \sum_{o\in\cO} G'(p^*)(o) p^*(o) +  \sum_{o\in\cO} G'(p^*)(o) q(o) \tag*{(\mbox{$p^*$} is the maximizer)}\\
    &= G(q) -  \sum_{o\in\cO} G'(q)(o) q(o) +  \sum_{o\in\cO} G'(q)(o) q(o) \tag*{(\mbox{$s$} is strictly proper so \mbox{$p^*=q$})}\\
    &= G(q).
\end{align*}

$(\Leftarrow)$

We define the (strictly) convex function $G$ using the expected score of some scoring rule $s$, i.e., $G(p)= \sum_{o\in\cO} s(p,o)p(o)$ and the subgradient $G'(p)=s(p,o)$. Then,
\begin{align*}
    G(p) - \sum_{o\in\cO} G'(p)(o) p(o) + G'(p)(o) &:=   \sum_{o\in\cO} s(p,o)p(o) -  \sum_{o\in\cO} s(p,o)p(o) + s(p,o)\\
                                    &= s(p,o)\\ 
\end{align*}
This implies that $s$ is a strictly proper scoring rule as a consequence of Theorem~\ref{theorem:gneiting}.
\end{proof}
\subsection{On existence of extreme points of $\cP$}
\label{appendix:existence-of-extreme-points}
If the set of probability distributions $\cP$ is infinite, the extreme points may not always exist. Therefore, we need to identify the conditions under which $\cP$ has a valid set of extreme points. We argue that for extreme points to exist for $\cP$, $\co(\cP)$ must equal $\co(\overline{\cP})$, where $\overline{\cP}$ is the closure of set $\cP$. 
\begin{table}[h]
    \centering
    \caption{A toy example on existence of $\ext(\cP)$ for $\cP$ made using $p_1,p_2\in\Delta(\cO)$}
    \begin{tabular}{c|c|c|c}
          \hline
          $\cP$ & Closed? & $ext(\cP)$ exists? & $\co(\cP)=\co(\overline{\cP})$\\
          \hline 
          $\{p_1,p_2\}$ & closed & yes & yes \\
         $\{p| p:=wp_1+(1-w)p_2\text{ }\forall w \in[0,1]-\{\frac{1}{2}\}\}$ & open & yes & yes\\
         $\{p| p:=wp_1+(1-w)p_2\text{ }\forall w \in(0,1)\}$ & open & no & no\\
         $\{p| p:=wp_1+(1-w)p_2\text{ }\forall w \in[0,1]\}$ & closed & yes & yes\\
         \hline 
    \end{tabular}
    \label{tab:my_label}
\end{table}

To show this we precisely define the extreme points of $\cP$, independent of the convex hull of $\cP$ as follows  
\begin{definition}
    Given a set $\cP$, we define the extreme points as $\ext(\cP)$ as the collection of $p\in \cP$ for which there does not exist a set of points $C\subseteq\cP\setminus\{p\}$ and a probability measure $w:\Delta\rightarrow[0,1]$ such that $p=\int_{C} w(q)dq$.
\end{definition}

For extreme points of a set to exist in general spaces, its convex hull must be compact according to Choquet's  Theorem~\cite{bishop1959representations}. To establish compactness of $\cP$ we first show that with an appropriate notion of distance $\Delta(\cO)$ can form a bounded metric space.  
\begin{proposition}
\label{prop:boundedspace}
    The metric space $(\Delta(\cO),d_{TV})$ is bounded, where $d_{TV}$ denotes the total variational distance between two probability distributions $p,q$ in terms of their corresponding probability measures $P,Q$ is defined as 
    \begin{equation*}
        d_{TV}(p,q):= \sup_{A\subseteq\cO}|P(A)-Q(A)|=\frac{1}{2}\int|p(o)-q(o)|do
    \end{equation*}
\end{proposition}
\begin{proof}
    As defined above, the total variational distance is half the L1 distance ~\cite{levin2017markov}. This allows us to express the total variation distance directly using densities. To show $(\Delta(\cO),d_{TV})$ is bounded, let $p,q$ be arbitrary distributions in $\Delta(\cO)$. 
    Then 
    \begin{align*}
        d_{TV}(p,q)&:= \frac{1}{2}\int|p(o)-q(o)|do\\
                    &< \frac{1}{2}\int |p(o)|+|-q(o)|do && (\text{ Triangle Inequality})\\
                    &= \frac{1}{2}\int |p(o)|do + \frac{1}{2}\int |q(o)| do \\
                    &= \frac{1}{2}\int p(o) do + \frac{1}{2}\int q(o) do && (p(o)\geq 0\text{ and } q(o)\geq 0)\\
                    &= \frac{1}{2}+\frac{1}{2} = 1
    \end{align*}
    Thus $(\Delta(\cO),d_{TV})$ is a bounded metric space.  
\end{proof}

We now discuss the conditions on $\cP$ such that $\ext(\cP)\subseteq\cP$. 
\begin{proposition}
\label{prop:whatconditionsonPforgoodbehaviour}
    There exists a probability measure $w\in\Delta(\ext(\cP))$ for all $p\in\cP$ such that 
    \begin{equation*}
        p = \int_{p\in\ext(\cP)} w(p)dp
    \end{equation*}
    iff $\co(\cP)=\co(\overline{\cP})$, where $\co(\overline{\cP})$ denotes the convex hull of the closure of $\cP$ when $\cO$ is finite. And for cases where $\cO$ is an infinite continuous set, $\co(\cP)$ must additionally be totally bounded.
\end{proposition}
\begin{proof}
    The above result is a direct implication of the Heine-Borel Theorem (Theorem 2.41, \citep{rudin1976}) and Choquet's theorem~\citep{bishop1959representations}. First we discuss the proof for the case where $\cO$ is finite. Since $\cP\subseteq\Delta(\cO)$, using ~\ref{prop:boundedspace} we can say that $\cP$ is bounded. This means that $\co(\cP)$ is also bounded. Now, we know that the convex hull of a closed set $\overline{\cP}$ is also closed. Therefore, $\co(\overline{\cP})$ is closed and since $\co(\cP)=\co(\overline{\cP})$, $\co(\cP)$ is also closed. This makes $\co(\cP)$ compact as it is both bounded and closed by Heine-Borel Theorem. Now we can directly apply Choquet's theorem to obtain a probability measure $w$ for every $p\in\cP$ such that $p = \int_{p\in\ext(\cP)} w(p)dp$. In case when $\cO$ is an infinite continuous set, we are dealing with $\cP\subseteq\Delta(\cO)$, where $\Delta(\cO)$ may not have Heine-Borel Property, thus $\co(\cP)$ being totally bounded in addition to closed ensures that $\co(\cP)$ is compact and therefore Choquet's theorem is applicable. 
\end{proof}
The proposition ~\ref{prop:whatconditionsonPforgoodbehaviour} tries to identify what conditions should $\cP$ satisfy so that we can interpret $\co(\cP)$, i.e., the convex hull of $\cP$ as a credal set with valid extreme points $\ext(\cP)$. The general condition is that $\co(\cP)=\co(\overline{\cP})$ as a condition on $\cP$. Equivalently, the condition on $\co(\cP)$ is that it is closed. Notice that for finite $\cP$, it is trivially satisfied. This allows us to exclude $\cP=(0,\frac{1}{2})$ type of open sets from our discussion since they are open sets and its convex hull will violate closedness i.e. $\co(\cP)=\cP=(0,\frac{1}{2})$. Depending on the convention, if credal sets for $\cP$ are defined as the closure of their convex hulls, i.e., $\overline{\co}(\cP)$, then credal sets are compact (Heine-Borel Theorem) and Proposition ~\ref{prop:whatconditionsonPforgoodbehaviour} is applicable. Thus for our , we will restrict our discussion to $\cP$ such that $\co(\cP)=\co(\overline{\cP})$.
\subsection{Lower and Upper probabilities are always extreme points}
\begin{lemma}
\label{lemma:pinfandpsupinextP}
Let $\cP\subseteq\Delta(\cO)$ be the forecaster's belief and $\ext(\cP)$ the extreme points of the convex hull generated by $\cP\subseteq\Delta(\cO)$. Given any scoring rule $s:2^{\Delta(\cO)}\times \cO\rightarrow\mathbb{R}$ and forecasts $\cQ\subseteq\Delta(\cO)$, let
\begin{equation*}
    p_{\text{L}}^{(s,\cQ)} := \arg\inf_{p\in\cP}\mathbb{E}_{p}[s(\cQ,o)], \quad\quad p_{\text{U}}^{(s,\cQ)} := \arg\sup_{p\in\cP}\mathbb{E}_{p}[s(\cQ,o)] .
\end{equation*}
Then, both $p_{\text{L}}^{(s,\cQ)}$ and $p_{\text{U}}^{(s,\cQ)}$ belong to $\ext(\cP)$ for all pairs of $s$ and $\cQ$. In addition, $\cP\simeq \ext(\cP)$.
\end{lemma}
\begin{proof}
Firstly, for all $p \in \cP$, either $p \in \ext(\cP)$ or $p \not\in \ext(\cP)$. This follows trivially from the definition of extreme points of a convex hull in section \ref{subsection:Imprecise-probability-and-credal-sets}. 

The proof proceeds as follows. In \textbf{(i)} and \textbf{(ii)}, we show that $p_{\text{L}}^{(s,\cQ)},p_{\text{U}}^{(s,\cQ)}\in \ext(\cP)$ for all pairs of $s$ and $\cQ$, respectively, with a contradiction. 
Then, given \textbf{(i)} and \textbf{(ii)}, $\cP\simeq \ext(\cP)$ follows from  Definition \ref{def:equivalenceofimpreciseforecasts} for the equivalence of imprecise beliefs.

\textbf{(i) Lower probability:} For all $s$ and $\cQ$, $p_{\text{L}}^{(s,\cQ)}\in \ext(\cP)$.

We prove this by contradiction. Let us first assume there exists a pair of $s,\cQ$ such that $p_{\text{L}}^{(s,\cQ)}\in \cP\setminus \ext(\cP)$. Since we have fixed $s$ and $\cQ$, we drop the superscript from $p_{\text{L}}^{(s,\cQ)}$ for readability and treat $p_{\text{L}}$ as a distribution in $\cP$. 
Next, given $p_{\text{L}}\in \cP\setminus \ext(\cP)$, it implies that there exists a second order distribution $w\in \Delta(\ext(\cP))$ such that $w(p)>0$ for all $p\in \ext(\cP)$.  
\begin{align*}
    p_{\text{L}} &= \int_{p\in \ext(\cP)}w(p)dp \\
(\implies)\quad \mathbb{E}_{p_{\text{L}}}[s(\cQ,o)]&=  \mathbb{E}_{\int_{p\in \ext(\cP)}w(p)dp}[s(\cQ,o)] \\
&= \int_{p\in \ext(\cP)}w(p)\mathbb{E}_{p}[s(\cQ,o)]dp\\
&> \inf_{p\in \ext(\cP)}\mathbb{E}_{p}[s(\cQ,o)]. && (w(p)>0 \text{ for all } p\in \ext(\cP))
\end{align*}
This results in a contradiction because $\ext(\cP)\subseteq \cP$. Therefore, $p_{\text{L}}\in ext(\cP)$.Since our choice of $\cQ$ and $s$ was arbitrary, the contradiction holds for all $\cQ\subseteq\Delta(\cO)$ and $s$. Therefore,  $p_{\text{L}}^{(s,\cQ)}\in \ext(\cP)$ for all and $s$.

\textbf{(ii) Upper probability:} For all $s$ and $\cQ$, $p_{\text{U}}^{(s,\cQ)}\in \ext(\cP)$.

Similarly, we show that $p_{\text{U}}\in \ext(\cP)$. Suppose that $p_{\text{U}}\in \cP\setminus \ext(\cP)$. This implies that there exists a second order distribution $w\in \Delta(\ext(\cP))$ such that $w(p)>0$ for all $p\in \ext(\cP)$ and 
\begin{align*}
    p_{\text{U}} &= \int_{p\in \ext(\cP)}w(p)dp\\
(\implies)\quad \mathbb{E}_{p_{\text{U}}}[s(\cQ,o)] &=  \mathbb{E}_{\int_{p\in \ext(\cP)}w(p)dp}[s(\cQ,o)] \\
&= \int_{p\in \ext(\cP)}w(p)\mathbb{E}_{p}[s(\cQ,o)]dp\\
&< \sup_{p\in \ext(\cP)}\mathbb{E}_{p}[s(\cQ,o)]. && (w(p)>0 \text{ for all } p\in \ext(\cP))
\end{align*}
This also results in a contradiction since $\ext(\cP)\subseteq \cP$. Hence, both $p_{\text{L}}$ and $p_{\text{U}}$ belong to $\ext(\cP)$. 

\textbf{Equivalence of $\cP$ and $\ext(\cP)$:} Next, we show that $\cP$ and $\ext(\cP)$ are equivalent by applying Definition \ref{def:equivalenceofimpreciseforecasts}. For any reported set of beliefs $Q\subseteq \Delta(\cO)$ and scoring rule $s$, 
\begin{align*}
    \text{im}(\bm{V}^{\cP}(\cQ))&=\left[\inf_{p\in \cP}\mathbb{E}_{p}[s(\cQ,o)], \;\sup_{p\in \cP}\mathbb{E}_{p}[s(\cQ,o)]\right]\\
    &= \left[\mathbb{E}_{p_{\text{L}}}[s(\cQ,o)], \; \mathbb{E}_{p_{\text{U}}}[s(\cQ,o)]\right]\\
    &= \left[\inf_{p\in \ext(\cP)}\mathbb{E}_{p}[s(\cQ,o)],\sup_{p\in \ext(\cP)}\mathbb{E}_{p}[s(\cQ,o)]\right]  && (p_{\text{L}},p_{\text{U}} \in \ext(\cP)\text{ and } \ext(\cP)\subseteq \cP)\\
    &= \text{im}(\bm{V}^{\ext(\cP)}(\cQ)).
\end{align*}
This completes the proof.
\end{proof}
\subsection{Equivalence of extreme points for elicitation}
\begin{lemma}
\label{lemma:equivalentPimpliessameexteremepoints}
If two beliefs $\cP,\cP'\subseteq\Delta(\cO)$ are equivalent, i.e., $\cP\simeq\cP'$, then $\ext(\cP)=\ext(\cP')$.    
\end{lemma}
\begin{proof}
    By Definition~\ref{def:equivalenceofimpreciseforecasts}, two imprecise beliefs $\cP,\cP'\subseteq\Delta(\cO)$ are equivalent if for all scoring rule $s$ and forecast $\cQ\subseteq\Delta(\cO)$, $\text{im}(\bm{V}^{\cP}(\cQ))=\text{im}(\bm{V}^{\cP'}(Q))$. This means that, 
    \begin{align*}
        \inf_{p\in\cP}\mathbb{E}_{p}[s(\cQ,o)]=\inf_{p'\in\cP'}\mathbb{E}_{p'}[s(\cQ,o)]\quad \text{ and }\quad \sup_{p\in\cP}\mathbb{E}_{p}[s(\cQ,o)]=\sup_{p'\in\cP'}\mathbb{E}_{p'}[s(\cQ,o)].
    \end{align*}
    ($\Rightarrow$) For the first part of the proof, we show that $\ext(\cP)\subseteq \ext(\cP')$. Let $q\in \ext(\cP)$, we know that for all $s,\cQ$,
    \begin{align*}
         \inf_{p\in \ext(\cP)}\mathbb{E}_{p}[s(\cQ,o)] & \leq \mathbb{E}_{q}[s(\cQ,o)]\leq\sup_{p\in \ext(\cP)}\mathbb{E}_{p}[s(\cQ,o)] \\
        \inf_{p\in\cP}\mathbb{E}_{p}[s(\cQ,o)] & \leq \mathbb{E}_{q}[s(\cQ,o)]\leq \sup_{p\in\cP}\mathbb{E}_{p}[s(\cQ,o)] && (\cP\simeq \ext(\cP) \text{ from Lemma}~\ref{lemma:pinfandpsupinextP})\\
        \inf_{p'\in\cP'}\mathbb{E}_{p'}[s(\cQ,o)] & \leq \mathbb{E}_{q}[s(\cQ,o)]\leq \sup_{p'\in\cP'}\mathbb{E}_{p'}[s(\cQ,o)] && (\cP\simeq\cP'\text{ by definition})\\
        \inf_{p'\in \ext(\cP')}\mathbb{E}_{p'}[s(\cQ,o)]&\leq \mathbb{E}_{q}[s(\cQ,o)]\leq \sup_{p'\in \ext(\cP')}\mathbb{E}_{p'}[s(\cQ,o)]. && (\cP'\simeq \ext(\cP') \text{ from Lemma}~\ref{lemma:pinfandpsupinextP})
    \end{align*}
    The last inequalities imply that $q\in \ext(\cP')$. ($\Leftarrow$) Next, we show that $\ext(\cP') \subseteq \ext(\cP)$. Let $q'\in \ext(\cP')$. Then, we know that for all $s,\cQ$,
    \begin{align*}
        \inf_{p'\in \ext(\cP')}\mathbb{E}_{p}[s(\cQ,o)]&\leq \mathbb{E}_{q'}[s(\cQ,o)]\leq \sup_{p'\in \ext(\cP')}\mathbb{E}_{p'}[s(\cQ,o)] \\
        \inf_{p'\in\cP'}\mathbb{E}_{p'}[s(\cQ,o)]&\leq \mathbb{E}_{q'}[s(\cQ,o)]\leq \sup_{p'\in\cP'}\mathbb{E}_{p}[s(\cQ,o)] && (\cP'\simeq \ext(\cP') \text{ from Lemma}~\ref{lemma:pinfandpsupinextP})\\
        \inf_{p\in\cP}\mathbb{E}_{p}[s(\cQ,o)]&\leq \mathbb{E}_{q'}[s(\cQ,o)]\leq \sup_{p\in\cP}\mathbb{E}_{p}[s(\cQ,o)] && (\cP\simeq\cP'\text{ by definition})\\
        \inf_{p\in \ext(\cP)}\mathbb{E}_{p}[s(\cQ,o)]&\leq \mathbb{E}_{q'}[s(\cQ,o)]\leq \sup_{p\in \ext(\cP)}\mathbb{E}_{p}[s(\cQ,o)]. && (\cP\simeq \ext(\cP) \text{ from Lemma}~\ref{lemma:pinfandpsupinextP})
    \end{align*}
    The last inequalities imply that $q' \in \ext(\cP)$. Since both $\ext(\cP)\subseteq \ext(\cP')$ and $\ext(\cP')\subseteq \ext(\cP)$, we can conclude that $\ext(\cP) = \ext(\cP')$.
\end{proof}

\subsection{Preference relation in the subset of a credal set}
The lemma argues that the dominance induced by the preference relation associated with a credal set can only be refined by considering its subsets. Formally, 
\begin{lemma} 
    \label{lemma:additional-subset-q-preference}
    For any pair of imprecise forecasts $\cP, \cQ\subseteq\Delta(\cO)$ such that $co(\cQ)\subset co(\cP)$
    \begin{align*}
      a\succeq_{\cP} a' \implies a \succeq_{\cQ} a' \quad \forall\text{ }a,a'\in\cA
    \end{align*}
    where $\succeq_\cP,\succeq_\cQ$ are the partial preference relations over the space of actions induced by the corresponding expected utility profiles $\{\mathbb{E}_p[u(\cdot,o)]\}_{p\in\cP}$ and $\{\mathbb{E}_q[u(\cdot,o)]\}_{q\in\cQ}$. 
\end{lemma}
\begin{proof}
    Let us assume an arbitrary $\cQ$ and $\cP$ such that $\cQ\subset \cP$. Now let us consider a pair of inputs $x,x'\in\cX$ such that $x\succeq_{\cP} x'$. This implies that  
    \begin{align*}
\mathbb{E}_p[u(x,o)]&\geq\mathbb{E}_p[u(x',o)]\quad \forall p\in\cP\\
      \implies \quad \mathbb{E}_p[u(x,o)]&\geq \mathbb{E}_p[u(x',o)] \quad \forall p\in co(\cP)\\
      \implies \quad \mathbb{E}_q[u(x,o)]&\geq \mathbb{E}_q[u(x',o)]\quad \forall q\in co(\cQ) &&(co(\cQ)\subset co(\cP))\\
      \implies \quad \mathbb{E}_q[u(x,o)]&\geq \mathbb{E}_q[u(x',o)] \quad \forall q\in \cQ\\
      \implies \quad x &\succeq_{\cQ} x'
    \end{align*}
\end{proof}
\section{Proof of Results in Section \ref{section:setup}}
\subsection{Proof of Proposition ~\ref{prop:equivalenceofpreciseforecasts}}
\begin{proof}
    ($\Leftarrow$) 
    Let us assume that there are two identical distributions $p,q\in\Delta(\cO)$, i.e., $p=q$ that implies $\mathbb{E}_{p}[s(\cQ,o)]=\mathbb{E}_{q}[s(\cQ,o)]$ for all $\cQ\subseteq\Delta(\cO)$ and IP scoring rule $s$. Therefore, $\{p\}\simeq\{q\}$. 

    ($\Rightarrow$) Next, let us assume that $\{p\}\simeq \{q\}$, which means that
    \begin{align*}
         \mathbb{E}_{p}[s(\cQ,o)]=\mathbb{E}_{q}[s(\cQ,o)], \quad \forall s, \forall \cQ \subseteq\Delta(\cO).
    \end{align*}
    Since the above holds for all $s$ and $\cQ$, we choose $s$ to be strictly proper for precise forecasts and $\cQ:=\{p\}$. Hence,
    \begin{align*}
        \mathbb{E}_{p}[s(\{p\},o)]&=\mathbb{E}_{q}[s(\{p\},o)]\\
       \implies \quad p&=q. &&(s\text{ is strictly proper for precise forecasts})
    \end{align*}
    This completes the proof.
\end{proof}

\subsection{Proof of Proposition ~\ref{prop:credalsets}}
\begin{proposition}
   For imprecise beliefs $\cP,\cP'\subseteq \Delta(\cO)$ with non-empty extreme points, $\cP\simeq \cP'$ if and only if $co(\cP)=co(\cP')$.
\end{proposition}
\begin{proof}

($\Leftarrow$) First, we assume that $\cP$ and $\cP'$ induce the same credal set, i.e., 
 \begin{align*}
     \co(\cP) &= \co(\cP') \\
      \ext(\cP) &= \ext(\cP') && \text{(credal sets are convex hulls)} \\
       \ext(\cP)\simeq\cP & \quad\text{ and }\quad \ext(\cP')\simeq\cP'&&\text{(Lemma \ref{lemma:pinfandpsupinextP})}\\
     \implies \quad \cP&\simeq \cP'
 \end{align*}
Hence, $\cP$ and $\cP'$ are equivalent. 

($\Rightarrow$) Next, we assume that $\cP$ and $\cP'$ are equivalent. Then, it follows from Lemma~\ref{lemma:equivalentPimpliessameexteremepoints} that $\ext(\cP)=\ext(\cP')$. 

Let us assume that there exists a $P\in\co(\cP)$. Since credal sets are convex sets, $P$ can be expressed as a convex combination of the extreme points. Therefore, there exists some $w\in\Delta(\ext(\cP))$ such that  
\begin{align*}
    P&= \int_{p\in\ext{\cP}}w(p)dp \underset{(\vardiamondsuit)}{=} \int_{p\in\ext{\cP'}}w(p)dp &&(\vardiamondsuit:\text{Lemma~\ref{lemma:equivalentPimpliessameexteremepoints}})
\end{align*}
Thus $P\in\co(\cP')$ and therefore, $\co(\cP')\subseteq\co(\cP)$. 

Similarly, let us assume that there exists a $P'\in\co(\cP')$. Now $P'$ can also be expressed as a convex combination of the extreme points. Therefore, there exists some $w'\in\Delta(\ext(\cP'))$ such that  
\begin{align*}
    P'&= \int_{p'\in\ext{\cP'}}w'(p')dp' \underset{(\vardiamondsuit)}{=} \int_{p'\in\ext{\cP}}w(p')dp' &&(\vardiamondsuit:\text{Lemma~\ref{lemma:equivalentPimpliessameexteremepoints}})
\end{align*}
Thus $P\in\co(\cP)$ and therefore, $\co(\cP)\subseteq\co(\cP')$. Since $\co(\cP')\subseteq\co(\cP)$ and $\co(\cP)\subseteq\co(\cP')$, therefore $\co(\cP)=\co(\cP')$
\end{proof}
\subsection{Proof of Theorem~\ref{theorem:nosaneproperscoringrule}}

\textbf{Part I:} We first show that for any IP scoring rule $s$, it must give a constant score to all forecasts.
\begin{proof}    
Let us assume there exists a proper scoring rule $s:2^{\Delta(\cO)}\times \cO\rightarrow\mathbb{R}\cup\{-\infty\}$. Then, according to the definition of proper IP scoring rules for an imprecise forecaster with a vacuous belief $\Delta(\cO)$, we must have, 
\begin{equation*}
    \Delta(\cO) \succeq_{\Delta(\cO)} \cQ, \quad \forall \cQ\not\simeq\Delta(\cO).
\end{equation*}
This follows from the fact that for a proper score $s$, $\bm{V}^{\Delta(\cO)}(\Delta(\cO))$ dominates $\bm{V}^{\Delta(\cO)}(\cQ)$. Consequently, it follows from Definition \ref{def:propernessofimprecisescoringrule} that
\begin{equation}\label{eq:vacuous-dominance}
    \mathbb{E}_p[s(\Delta(\cO),o)] \geq \mathbb{E}_p[s(\cQ,o)], \quad \forall \cQ\not\simeq\Delta(\cO),\; \forall p\in\Delta(\cO).
\end{equation}
Let $\tilde{\cQ} := \{\cQ \,|\, \cQ\not\simeq\Delta(\cO)\}$ be the set of all forecasts not equivalent to the forecaster's belief ($\Delta(\cO)$), then we can rewrite \eqref{eq:vacuous-dominance} as 
    \begin{align}
        \label{eq:theorem-1-eq-1}
        \mathbb{E}_p[s(\Delta(\cO),o)]&\geq  \mathbb{E}_p[s(\cQ,o)], \quad \forall \cQ\in\tilde{\cQ},\; \forall p\in\Delta(\cO).
    \end{align}
Also, $\{q\}_{q\in\Delta(\cO)}\subseteq\tilde{\cQ}$ since $q\not\simeq\Delta(\cO)$. Combining this with Equation~\eqref{eq:theorem-1-eq-1} yields
\begin{align}
    \mathbb{E}_p[s(\Delta(\cO),o)] &\geq  \mathbb{E}_p[s(\{q\},o)], \quad \forall q\in\Delta(\cO),\; \forall p\in\Delta(\cO) \nonumber \\
    \label{eq:theorem-1-eq-2}
    \implies \quad \mathbb{E}_p[s(\Delta(\cO),o)]&\geq  \mathbb{E}_p[s(\{p\},o)], \quad \forall p\in\Delta(\cO), && 
\end{align}
where the second inequalities follow by selecting the inequalities such that $q=p$.
Similarly, let us analyse the incentives for all precise forecasters with belief $p\in\Delta(\cO)$ given a proper IP scoring rule $s$. 
Then, for all precise forecasters we must have, 
\begin{align*}
    \{p\}&\succeq_{\{p\}} \cQ, \quad \forall p\in \Delta(\cO),\; \forall \cQ\in\tilde{\cQ} && (\tilde{\cQ}:=2^{\Delta(\cO)}\setminus \{p\})\\
    \implies \quad \{p\}&\succeq_{\{p\}} \Delta(\cO), \quad \forall p \in \Delta(\cO) && (\Delta(\cO)\in \tilde{\cQ})\\
    \implies \quad \mathbb{E}_p[s(\{p\},o)]&\geq  \mathbb{E}_p[s(\Delta(\cO),o)], \quad \forall p\in\Delta(\cO).
\end{align*}    
However, it follows from Equation \eqref{eq:theorem-1-eq-2} that $\mathbb{E}_p[s(\Delta(\cO),o)]\geq \mathbb{E}_p[s(\{p\},o)]$ and $\mathbb{E}_p[s(\{p\},o)]\geq\mathbb{E}_p[s(\Delta(\cO),o)]$ for all $p\in\Delta(\cO)$. This implies that $\mathbb{E}_p[s(\Delta(\cO),o)]=\mathbb{E}_p[s(\{p\},o)]$ for all $p\in\Delta(\cO)$.

Therefore, any IP scoring rule $s$ that satisfies properness sets up incorrect incentives for the forecaster. For example, the expected score for honestly reporting a precise forecast is the same as reporting the vacuous set of all distributions, i.e., 
\begin{equation}
    \label{eq:theorem-1-eq-3}
    \mathbb{E}_p[s(\Delta(\cO),o)]=\mathbb{E}_p[s(\{p\},o)], \quad \forall p\in\Delta(\cO).
\end{equation}
While the above equation is sufficient to discard any proper scoring rule, we show that the only IP scoring rule possible is a constant function. For $s$ to be proper for imprecise forecasts, the following must hold true for all $\cP\subseteq\Delta(\cO)$:
\begin{align}
    \cP&\succeq_\cP \{q\}, \quad \forall q\in \Delta(\cO)\nonumber \\
    \mathbb{E}_p[s(\cP,o)]&\geq\mathbb{E}_p[s(\{q\},o)], \quad \forall q\in \Delta(\cO),\; \forall p\in\cP \nonumber \\
    \mathbb{E}_p[s(\cP,o)]&\geq\mathbb{E}_p[s(\{q\},o)], \quad \forall q\in \cP, \; \forall p\in\cP \nonumber \\
    \label{eq:theorem-1-eq-4}
    \implies \mathbb{E}_p[s(\cP,o)]&\geq\mathbb{E}_p[s(\{p\},o)], \quad \forall p\in\cP.
\end{align}
Similarly, for any $p\in\Delta(\cO)$, the following must hold: 
\begin{align}
    \{p\}&\succeq_p \cP, \quad \forall p\in \Delta(\cO) \nonumber \\
    \label{eq:theorem-1-eq-5}
    \implies \mathbb{E}_p[s(\{p\},o)]&\geq \mathbb{E}_p[s(\cP,o)].
\end{align}
Combining Equations~\ref{eq:theorem-1-eq-3}, ~\ref{eq:theorem-1-eq-4} and ~\ref{eq:theorem-1-eq-5} yields 
\begin{align}
\label{eq:theorem-1-eq-6}
    \mathbb{E}_p[s(\Delta(\cO),o)]&=\mathbb{E}_p[s(\{p\},o)]=\mathbb{E}_p[s(\cP,o)], \quad \forall p\in\Delta(\cO)
\end{align}
Given Equation~\ref{eq:theorem-1-eq-6} is valid for all $p\in\Delta(\cO)$, we consider the a subset of $\Delta(\cO)$. To be precise, the set of all Dirac distributions associated with each outcome, i.e. $p\in\{\delta_{o}\}_{o\in\cO}$
\begin{align*}
    \mathbb{E}_p[s(\Delta(\cO),o)]&=\mathbb{E}_p[s(\{p\},o)]=\mathbb{E}_p[s(\cP,o)], \quad p\in \{\delta_o\}_{o\in \cO}  && (\{\delta_o\}_{o\in \cO}\subseteq \Delta(\cO))\\
    \implies s(\Delta(\cO),o)&=s(\{p\},o)=s(\cP,o), \quad \forall o\in \cO.
\end{align*}
Hence, $s$ needs to be a constant score for it to be a proper IP scoring rule.  
\end{proof}
\textbf{Part II:} There exists no strictly proper IP scoring rule $s$.
\begin{proof}
Assume that there exists a strictly proper IP scoring rule $s$. Consider a precise forecaster with belief $q\in\Delta(\cO)$. Then, we have 
\begin{align}
    \{q\}\succ_{q} \cQ, \quad \forall \cQ\not\simeq{q} \nonumber \\
    \implies \mathbb{E}_{q}[s(\{q\},o)]> \mathbb{E}_{q}[s(\cQ,o)] \nonumber \\
    \label{eq:theorem-1-eq-7}
    \implies \mathbb{E}_{q}[s(\{q\},o)]> \mathbb{E}_{q}[s(\Delta(\cO),o)]. && (\Delta(\cO)\text{ is one possible }\cQ)
\end{align}
Since $s$ is strictly proper, it satisfies Equation~\ref{eq:theorem-1-eq-3}. However, this results in a contradiction to Equation ~\ref{eq:theorem-1-eq-7}. Hence, no $s$ can be strictly proper.  
\end{proof}
\section{Proof of Results in Section \ref{sec:imprecisescoringrules}}

\subsection{Why is Non-dictatorship Desirable?}
\label{appendix:dictatorships}
Let us assume that $\rho$ violates non-dictatorship, then $\rho$ is dictatorial. For clarity, we also define a dictatorship.
\begin{definition}(Dictatorship)
An aggregation rule $\rho$ is a dictatorial if there exists a $P_\rho\in\cP$ (dictator), that depends on $\rho$, such that for any pair of reports $\cQ,\cQ'\subseteq\Delta(\cO)$,
\begin{equation*}
    \cQ\succeq_{P_{\rho}} \cQ' \implies \cQ\succeq_{\rho[\cP]}\cQ'.
\end{equation*}
\end{definition}
A dictatorial $\rho$ not only allows the forecaster to remove indecision in their decision-making problem about which $\cQ$ to report, it also allows the forecaster to precisely resolve their epistemic uncertainty, i.e., by reducing the credal set $\cP$ to only the dictator $P_\rho$. 

Let us denote the set of best reports plausible under aggregation $\rho$ by $\tilde{\cQ}^{\rho}:=\{\cQ \,|\, \cQ\succeq_{\rho} \cQ', \forall \cQ'\subseteq\Delta(\cO)\}$. Since $\succeq_{\rho}$ is complete, if the set of best reports $\tilde{\cQ}$ contains more than one report, then they must be indifferent w.r.t. $\succeq_{\rho[\cP]}$. 
Given $\rho$ is a dictatorship, there exists $P_\rho\in\cP$ such that $\succeq_{P_\rho}$ dictates the preference $\succeq_{\rho[\cP]}$. That is, the set of best reports under $P_\rho$ must be exactly the same as that under $\rho$. Therefore,  
\begin{equation*}
    \tilde{\cQ}^{P}=\tilde{\cQ}^{\rho}.
\end{equation*}
This implies that the expected scores of ${P_\rho}$ and $\cP$ with any dictatorial $\rho$ is the same, i.e.,  
\begin{equation*}
    V_{\rho}^\cP(\{P_\rho\}) = V_{\rho}^\cP(\cP).
\end{equation*}

\subsection{Proof of Proposition~\ref{prop:tailoredscoringrule}}
\begin{proof}
    We prove this result by contradiction. Let us assume that there exists a tailored scoring rule $s_\rho$ that is not proper and analyse this scoring rule for an arbitrary forecaster with an imprecise belief $\cP\subseteq\Delta(\cO)$. Since $s_\rho$ is not proper, it implies that there exists $Q\subseteq\Delta(\cO)$ where  $\cQ\not\simeq\cP$ such that 
    \begin{equation}
    \label{eq:proposition-4-point-9-eq-1}
        V_{\rho}^{\cP}(\cQ)> V_{\rho}^{\cP}(\cP).
    \end{equation}
    In other words, the forecaster strictly prefers the forecast $\cQ$ over their belief $\cP$. However, let's analyse the scenario from DM's perspective when they obtain forecast $\cP$, the optimal action according to the forecast $\cP$ is
    \begin{align}
    \label{eq:proposition-4-point-9-eq-2}
    a^*_{\cP,\rho} = \arg\max_{a\in\cA} \; \rho\left[\{\mathbb{E}_{p}[u(a,o)]\}_{p\in\cP}\right].
    \end{align}
    Since $a^*_{\cP,\rho}$ is the maximizer of DM's aggregated utility, this means that for all $a\in\cA$,
    \begin{align}
    \label{eq:proposition-4-point-9-eq-3}
        \rho\left[\{\mathbb{E}_{p}[u(a^*_{\cP,\rho},o)]\}_{p\in\cP}\right] \geq \rho\left[\{\mathbb{E}_{p}[u(a,o)]\}_{p\in\cP}\right].
    \end{align}
    However, we know that from \Cref{eq:proposition-4-point-9-eq-1}
    \begin{align*}
        V_{\rho}^{\cP}(\cQ)&> V_{\rho}^{\cP}(\cP)\\
        \rho[\{\mathbb{E}_{p}[s_{\rho}(\cQ,o)]\}_{p\in\cP}]&>\rho[\{\mathbb{E}_{p}[s_\rho(\cP,o)]\}_{p\in\cP}]\\
        \rho[\{\mathbb{E}_{p}[u(a^*_{\cQ,\rho},o)]\}_{p\in\cP}]&>\rho[\{\mathbb{E}_{p}[u(a^*_{\cP,\rho},o)]\}_{p\in\cP}].
    \end{align*}
    This results in a contradiction to \Cref{eq:proposition-4-point-9-eq-3}. Therefore, $s_\rho$ must be proper. Since this holds for any choice of $\rho$, we can conclude that $s_\rho$ must be proper for any aggregation rule $\rho$.
\end{proof}

\subsection{Proof of Lemma \ref{lemma:strictness-for-precise-distributions}}

\textbf{Part I: Strict properness of IP scoring rule for precise forecasts}

\begin{proof} ($\Rightarrow$) From \Cref{theorem:gneiting}, a regular precise scoring rule $s$ is (strictly) proper if and only if there exists a corresponding (strictly) convex function $G$ on $\Delta(\cO)$ such that 
    \begin{align*}
        s(p,o)= G(p) - \sum_{o\in\cO} G'(p)(o) p(o) + G'(p)(o).
    \end{align*}
    Moreover, it follows from \Cref{remark:gneiting} that the $G(p)=\mathbb{E}_{p}[s(p,o)]$. Hence, for a tailored scoring rule $s_\rho$ on precise distribution $p\in\Delta(\cO)$ to be strictly proper, we must have 
    \begin{align*}
        G(p)&=\mathbb{E}_{p}[s_\rho(\{p\},o)]\\
            &= k\mathbb{E}_{p}[u(a^*_p,o)]+c && (\text{Tailored scoring rule; \Cref{def:tailoredscoringrule}})\\
            &= \max_{a\in\cA}\; k\mathbb{E}_{p}[u(a,o)]+c .
    \end{align*}
    Next, for $G(p)$ to be strictly convex in $p$, we must have that for all $p,q\in\Delta(\cO)$,
    \begin{align}
    \label{eq:lemma-equation-2}
         G(q) &> G(p)+\sum_{o\in\cO} G'(p)(o)[q(o)-p(o)] 
     \end{align}
    Where $G'(p)(o)$ is the $o^{\text{th}}$ component of the gradient $G'(p)$ at $p$. Let us consider the right-hand side of \Cref{eq:lemma-equation-2}.
    \begin{align}
          G(p)+\sum_{o\in\cO} G'(p)(o)[q(o)-p(o)]  &= k\mathbb{E}_{p}[u(a^*_p,o)]+ c + \sum_{o\in\cO} ku(a^*_p,o) [q(o)-p(o)]\\
             &= k\mathbb{E}_{p}[u(a^*_p,o)]+ k\mathbb{E}_{q}[u(a^*_p,o)]-k\mathbb{E}_{p}[u(a^*_p,o)]+c\\
        \label{eq:lemma-equation-3}
             &= k\mathbb{E}_{q}[u(a^*_p,o)]+c .
    \end{align}
    Since $G(q):= k\mathbb{E}_{q}[u(a^*_q,o)]+c$, for $G$ to be strictly convex, we use \Cref{eq:lemma-equation-3} to rewrite \Cref{eq:lemma-equation-2} as follows
    \begin{align*}
          G(q) &> k\mathbb{E}_{q}[u(a^*_p,o)]+c, \quad \forall p,q\in\Delta(\cO), \\
        \implies \quad k\mathbb{E}_{q}[u(a^*_q,o)]+c &> k\mathbb{E}_{q}[u(a^*_p,o)]+c, \quad \forall p,q\in\Delta(\cO) .
    \end{align*}
    Hence, $a^*_q$ must be a unique maximizer.

    ($\Leftarrow$) 
    
    We assume that $a^*_{p} := \arg\max_{a\in\cA}\mathbb{E}_{p}[u(a,o)]$ is the unique maximizer for all $p\in\Delta(\cO)$. Then, for all $p,q\in\Delta(\cO)$ and some arbitrary $\lambda\in[0,1]$,
    \begin{align*}
        G(\lambda p+(1-\lambda)q)&= \mathbb{E}_{\lambda p+(1-\lambda)q}[s_\rho(\{\lambda p+(1-\lambda)q\},o)]\\
                                 &= \lambda \mathbb{E}_p[s_\rho(\{\lambda p+(1-\lambda)q\},o)]+(1-\lambda)\mathbb{E}_q[s_\rho(\{\lambda p+(1-\lambda)q\},o)]\\
                                 &= \lambda k\mathbb{E}_p[u(a^*_{\lambda p+(1-\lambda)q},o)]+\lambda c+(1-\lambda)k\mathbb{E}_q[u(a^*_{\lambda p+(1-\lambda)q},o)]+(1-\lambda)c \\
                                 &< \lambda k\mathbb{E}_p[u(a^*_{p},o)]+\lambda c+(1-\lambda)k\mathbb{E}_q[u(a^*_{q},o)]+(1-\lambda)c \tag*{(\mbox{$a^*_p$} and \mbox{$a^*_q$} are unique)}\\
                                 &=\lambda G(p)+(1-\lambda)G(q).
    \end{align*}
    Hence, $G$ is strictly convex.
\end{proof}
\textbf{Part II: Impossibility of strictly proper scoring rules with Pareto efficient $\rho$}
\begin{proof}
Suppose that there exists the aggregation rule $\rho$ such that the tailored scoring rule $s_\rho$ is strictly proper for both precise and imprecise forecasts. This means that for all $\cP\subseteq\Delta(\cO)$, and for all $\cQ\not\simeq\cP$,
\begin{align*}
    V_\rho^\cP(\cP)&>V_\rho^\cP(\cQ)\\
    \implies \quad \rho(\{\mathbb{E}_p[s_\rho(\cP,o)]\}_{p\in\cP})&> \rho(\{\mathbb{E}_p[s_\rho(\cQ,o)]\}_{p\in\cP})\\
    \implies \quad \rho(\{\mathbb{E}_p[u(a^*_{\rho,\cP},o)]\}_{p\in\cP})&> \rho(\{\mathbb{E}_p[u(a^*_{\rho,\cQ},o)]\}_{p\in\cP}). & (s_\rho \text{ is tailored scoring rule})
\end{align*}
The aggregation rule $\rho$ maps the set of preferences $\succeq_{\cP}:=\{\succeq_p\}_{p\in\cP}$ into a complete preference relation $\succeq_{\rho(\cP)}$ which follows the aggregated utility $\rho(\{\mathbb{E}_p[u(\cdot,o)]\}_{p\in\cP})$.

Since $\rho$ is Pareto efficient, for all $a,a'\in\cA$, $a\succeq_{\cP} a'$ implies $a\succeq_{\rho[\cP]} a'$. Only for actions that are incomparable to one another, i.e., $a\not\succeq_{\cP}a'$ and $a'\not\succeq_{\cP}a$, $\rho$ decides to remove indecision by completing the preference as $a\succeq_{\rho[\cP]}a'$ or $a'\succeq_{\rho[\cP]}a$. 

Without loss of generality, let us assume that $\rho$ chooses to rank $a\succeq_{\rho[\cP]}a'$ for two incomparable $a,a'\in\cA$ with respect to original credal set $\cP$.
However, based on \Cref{lemma:additional-subset-q-preference}, we can construct a $Q\subseteq\Delta(\cO)$ such that $\co(\cQ)\subset \co(\cP)$ and $a^*_{\rho,\cP}=a^*_{\rho,\cQ}$. This provides a counterexample to strictness of $s_\rho$ for all Pareto efficient $\rho$.

We now explain the counterexample in detail. We construct $\cQ$ based on its partial preference relation $\succeq_Q$. The preference relation $\succeq_Q$ must be well defined for any two pair of actions $a,a'\in\cA$. To this end we use the preference relation $\succeq_{\cP}$ to define all possible scenarios for a pair of actions $a,a'\in\cA$. Either $a,a'\in\cA$ are comparable with respect to $\succeq_{\cP}$ (\textbf{Case I}) or incomparable (\textbf{Case II}). The construction of $\succeq_{\cQ}$ is defined below
\usetikzlibrary{shapes.geometric, arrows}
\begin{figure}[h]
\centering
\begin{tikzpicture}[node distance=2cm, thick, >=stealth]
    
    \node (case1) [] {\textbf{Case I}};
    \node (process1) [below=0.5cm of case1] {$a\succeq_{\cP}a'$ implies $a\succeq_{\cQ}a'$};
    \node (end1) [below=0.1cm of process1] {(\Cref{lemma:additional-subset-q-preference})};
    \node (case2) [right of=case1, xshift=5cm] {\textbf{Case II}};
    \node (decision) [below=0.5cm of case2] {$(a\not\succeq_{\cP}a')\land(a'\not\succeq_{\cP}a)$};
    \node (branch1) [below left of=decision, xshift=-1cm] {Case II.1};
    \node (branchend1) [below=0.05cm of branch1] {$a\succeq_{\cQ} a':=a\succeq_{\rho[\cP]}a'$};
    \node (branch2) [below right of=decision, xshift=1cm] {Case II.2};
    \node (branchend2) [below=0.05cm of branch2] {$(a\not\succeq_{\cQ}a')\land(a'\not\succeq_{\cQ}a)$};
    \draw[->] (decision) --  (branch1);
    \draw[->] (decision) --  (branch2);
\end{tikzpicture}
\caption{In \textbf{Case I} when actions are comparable in $\succeq_{\cP}$ the \Cref{lemma:additional-subset-q-preference} dictates their order to be the same for partial preference induced by $\succeq_{\cQ}$. However, in \textbf{Case II} when the actions are incomparable w.r.t $\succeq_{\cP}$ either their order in $\succeq_{\cQ}$ must be set to aggregated order of $\cP$ i.e. $\succeq_{\rho[\cP]}$ or they are left untouched, i.e. incomparable w.r.t $\succeq_{\cQ}$}
\end{figure}
Now we are ready to reason what happens when we aggregate the partial preference $\succeq_{\cQ}$ with $\rho$. We will reason for all the cases we defined above. 
\begin{itemize}[label={}]
    \item \textbf{Case I:} For all pairs of $ a,a' \in\cA$ that are comparable w.r.t. $\succeq_{\cP}$  (Assume w.l.o.g $a\succeq_\cP a'$).
\begin{align*}
    a\succeq_\cP a' \underset{(\clubsuit)}{\implies}   a\succeq_{\rho[\cP]} a'\quad \text{ and } \quad  a\succeq_\cP a' \underset{(\vardiamondsuit)}{\implies}   a\succeq_\cQ a' \underset{(\clubsuit)}{\implies}   a\succeq_{\rho[\cQ]} a'. \tag*{(\mbox{$\clubsuit$}:\mbox{$\rho$} is PE,\mbox{$\vardiamondsuit$}: \Cref{lemma:additional-subset-q-preference})}
\end{align*}
Therefore, whenever the pair of actions $a,a' \in\cA$ are comparable w.r.t. $\succeq_{\cP}$, the aggregated preference relation is the same, i.e., $\{\succeq_{\rho[\cP]}\}\equiv\{\succeq_{\rho[\cQ]}\}$.
    \item \textbf{Case II:} Consider $a,a' \in\cA$ that are incomparable w.r.t. $\succeq_{\cP}$. (Assume w.l.o.g that $\rho$ resolves this as $a\succeq_{\rho[\cP]}a'$)
\begin{itemize}[label={}]
    \item \textit{Case II.1}: The pair of $ a,a' \in\cA$ is also comparable w.r.t. $\succeq_{\cQ}$ (Assume w.l.o.g $a\succeq_{\cQ} a'$)
 \begin{align*}
    a\succeq_\cQ a' \underset{(\clubsuit)}{\implies}   a\succeq_{\rho[\cQ]} a'\quad \text{ and } \quad  a\succeq_\cQ a'\underset{(\vardiamondsuit)}{\implies} a\succeq_{\rho[\cP]} a'  \tag*{(\mbox{$\clubsuit$}:\mbox{$\rho$} is PE, \mbox{$\vardiamondsuit$}: by construction)}
\end{align*}
\item \textit{Case II.2}: The pair $ a,a' \in\cA$ is incomparable w.r.t. $\succeq_{\cQ}$.

Since $\rho$ is a function, it will resolve indecision for two inputs in the same way, given that $|\cA|$ is fixed across both these resolutions:  
\begin{align*}
    \Bigg((a\not\succeq_{\cP}a' \land a'\not\succeq_{\cP}a)\implies (a\succeq_{\rho[\cP]} a') \Bigg) \land (a\not\succeq_{\cQ}a' \land a'\not\succeq_{\cQ}a)\implies a\succeq_{\rho[\cQ]} a'.
\end{align*}
\end{itemize}
Therefore, similar to Case 1, whenever the pair of actions $a,a' \in\cA$ are incomparable w.r.t. $\succeq_{\cP}$, the aggregated preference is the same, i.e., $\{\succeq_{\rho[\cP]}\}\equiv\{\succeq_{\rho[\cQ]}\}$. Hence, $a^*_{\rho,\cP}=a^*_{\rho,\cQ}$.
\end{itemize}
This makes $s_\rho$ not strictly proper.
\end{proof}

\section{Proof for Results in Section \ref{sec:rand-tailored-rules}}

\subsection{Proof of Lemma ~\ref{lemma:unqiue-representation-of-credal-set}}

\begin{proof}
     We prove this by contradiction, let us assume that, $\cA_{\ext}$ is not a sufficient way to represent credal sets in the actions space. This implies that there exists a pair of credal sets $\cQ,\cQ'\subseteq\cO$ such that $\cQ\not\simeq\cQ'$ and $\cA^{\ext}_{\cQ'}=\cA^{\ext}_\cQ$. Since $\cQ\not\simeq\cQ'$ it implies either of the two cases 
    \begin{itemize}[]
        \item \textbf{Case 1: } There exists a $ q'\in \ext(\cQ')$ such that $q'\not\in \ext(\cQ)$. This implies that 
     \begin{align*}
         \exists \text{ }a^*_{q'}\in \cA^{\ext}_{\cQ'}\quad\text{and}\quad \exists\text{ } a^*_{q'}\in \cA^{\ext}_{\cQ} && (a^*_{q'}\text{ is unique for all } q'\in\Delta(\cO))
     \end{align*}
     This results contradicts $\cA^{\ext}_{\cQ'}=\cA^{\ext}_\cQ$.
     \item \textbf{Case 2: } There exists a $q \in \ext(\cQ)$ such that $q\not\in \ext(\cQ')$. We follow the same reasoning as Case 1, i.e., 
     \begin{align*}
         \exists\text{ } a^*_{q}\in \cA^{\ext}_{\cQ}\quad\text{and}\quad \exists\text{ } a^*_{q}\in \cA^{\ext}_{\cQ'} && (a^*_{q}\text{ is unique for all } q\in\Delta(\cO))
     \end{align*}
     resulting in a contradiction with $\cA^{\ext}_{\cQ'}=\cA^{\ext}_\cQ$.
    \end{itemize}
    Hence $\cA^{\ext}$ is a unique representation for all credal sets. 
\end{proof}

\subsection{Proof of Theorem ~\ref{theorem:strictly-proper-imprecise-scoring-rule}}

\begin{proof}
We know that for $s_{\theta}$ to be strictly proper, the following must hold for all beliefs $\cP\subseteq\Delta(\cO)$ 
\begin{align*}
    V^{\cP}_{\theta}(\cP)=V^{\cP}_{\theta}(\cQ) \quad \text{iif} \quad \cP\simeq \cQ 
\end{align*}

$(\Rightarrow)$ Given, $\theta\in\Delta(\bm{\rho})$ has full support and $V^{\cP}_{\theta}(\cP)=V^{\cP}_{\theta}(\cQ)$. We know that $\Delta(\rho)$ is positive everywhere and $\rho(\{\mathbb{E}_p[s_{\rho}(\cP,o)]\}_{p\in\cP})- \rho(\{\mathbb{E}_p[s_{\rho}(\cQ,o)]\}\geq 0$ for all $\rho$ as $s_\rho$ is proper for all aggregations $\rho$ based on \Cref{prop:tailoredscoringrule}. It implies that,
\begin{align*}
       \rho(\{\mathbb{E}_p[s_{\rho}(\cP,o)]\}_{p\in\cP}) &= \rho(\{\mathbb{E}_p[s_{\rho}(\cQ,o)]\}_{p\in\cP}) \quad \forall \rho\in \bm{\rho}\quad \forall \cP \subseteq\Delta(\cO)\\
      \bm{\lambda}^{\top}\{\mathbb{E}_p[u(a^*_{\bm{\lambda}^{\top}\cP},o)]\}&=\bm{\lambda}^{\top}\{\mathbb{E}_p[u(a^*_{\bm{\lambda}^{\top}\cQ},o)]\}\quad \forall \bm{\lambda} \in \Delta^{|ext(\cP)|} \quad \forall \cP\subseteq\Delta(\cO) \tag*{(\mbox{$\bm{\rho}$}: fixed linear aggregation)}\\
      \bm{\lambda}^{\top}\{\mathbb{E}_p[u(a^*_{\bm{\lambda}^{\top}\cP},o)]\} &=\bm{\lambda}^{\top}\{\mathbb{E}_p[u(a^*_{\bm{\lambda}^{\top}\cQ},o)]\} \quad \forall \bm{\lambda} \in\{\delta_i\}_{i\in |ext(\cP)|}\quad \forall \cP\subseteq\Delta(\cO)\tag*{\mbox{$\{\delta_i\}_{i\in |ext(\cP)|} \subset \Delta^{|ext(\cP)|}$})}\\
      \mathbb{E}_p[u(a^*_{p},o)]&=\mathbb{E}_p[u(a^*_{q},o)]\quad \forall p\in \cP \quad \forall \cP\subseteq\Delta(\cO) && \tag*{(\mbox{$q:= \delta_i^T\cQ$})}\\
       a^*_{p} &= a^*_{q} \quad \forall p\in \cP \quad \forall \cP\subseteq\Delta(\cO)\tag*{(Using \Cref{lemma:strictness-for-precise-distributions} as \mbox{$s_\theta$} is strictly proper for precise distributions)}\\
    \implies \cA^{ext}_\cP &= \cA^{ext}_\cQ  \tag*{(By Definition of  \mbox{$\cA^{ext}$})}\\
    \implies \cP &\simeq \cQ \tag*{(\Cref{lemma:unqiue-representation-of-credal-set})}
\end{align*}

$(\Leftarrow)$ Given that $\cP\simeq \cQ$ we show that $V_{\theta}^{\cP}(\cP)=V_{\theta}^{\cP}(\cQ)$. This is trivial since two equivalent forecasts produce the same underlying partial order on the actions $\cA$. As aggregation functions make this partial order complete, by the property of being a function, they will result in the same complete order for the same partial order. Therefore, given $\cP\simeq\cQ$ implies that 
\begin{align*}
      V^{\cP}_\rho(\cP)&=V^{\cP}_\rho(\cQ) \quad \forall \rho\in\bm{\rho}\\
       \mathbb{E}_{\theta} [V^{\cP}_\rho(\cP)] &= \mathbb{E}_{\theta} [V^{\cP}_\rho(\cQ)] \quad \forall \theta \in \Delta(\bm{\rho})\\
       V^{\cP}_{\theta}(\cP)&=V^{\cP}_{\theta}(\cQ) \quad \forall \theta \in \Delta(\bm{\rho})
\end{align*}
Therefore, the imprecise forecaster is truthful in the epistemic sense w.r.t the strictly proper IP scoring rule $s_{\theta}$.  
\end{proof}
\section{Simulations}
\label{appendix:simulation}
To test the sanity of our proposed scoring rule, we simulate a scenario where an imprecise forecaster predicts a binary outcome (e.g., chance of rain tomorrow). We assume the forecaster has an imprecise forecast $[0.4,0.6]$ and uses an imprecise scoring rule $s_\rho$ where $\rho$ is a dictatorship or some other aggregation like min-max. We compare this to our randomized imprecise scoring rule $s_\theta$. Given the binary outcome, the forecaster reports an interval $\cQ:=[q_1,q_2]$ where $q_1$ denotes the lower probability and $q_2$ the upper probability respectively.  
\begin{figure}[t]
    \centering
    \begin{subfigure}{0.32\linewidth}
        \centering
        \includegraphics[width=\linewidth]{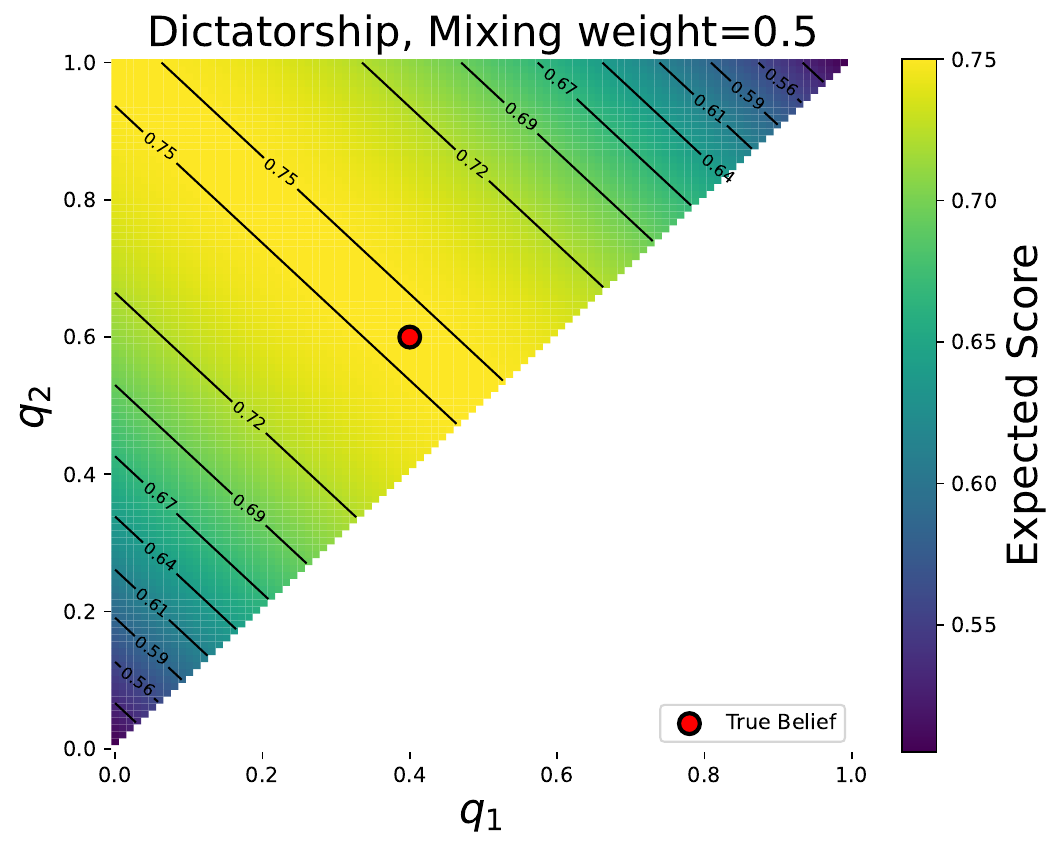}
        \caption{}
    \end{subfigure}
    \begin{subfigure}{0.32\linewidth}
        \centering
        \includegraphics[width=\linewidth]{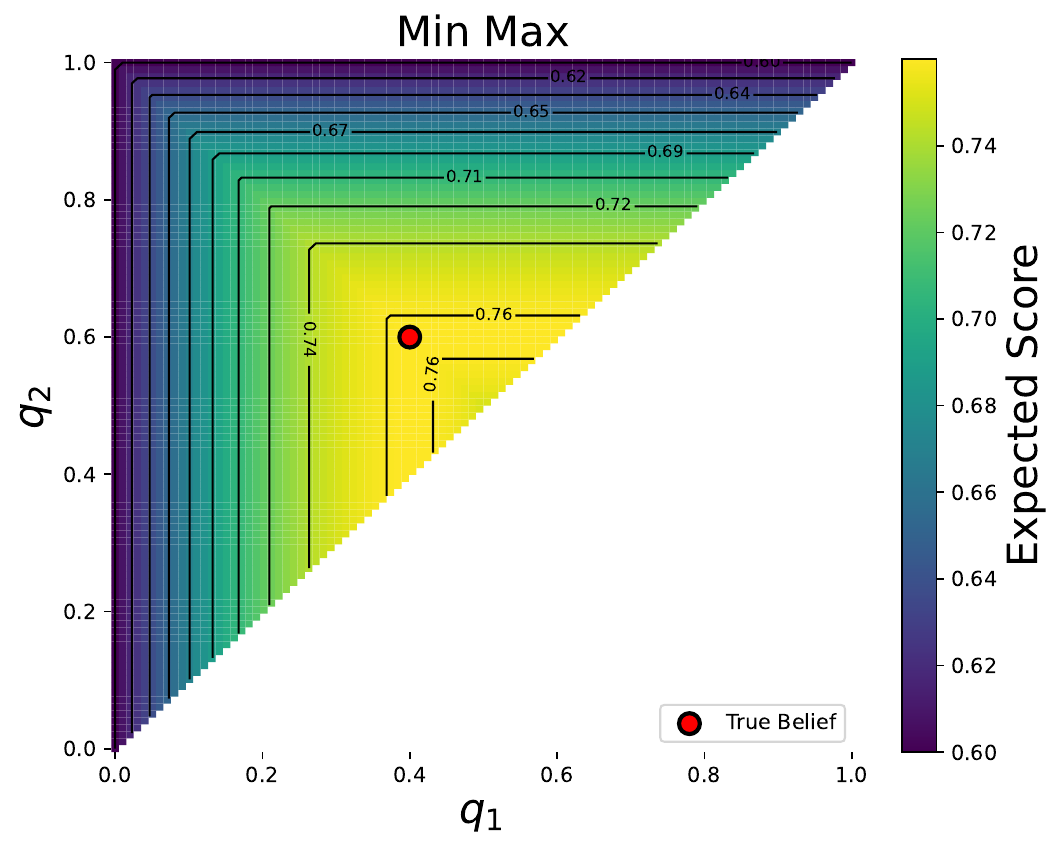}
        \caption{}
    \end{subfigure}
    \begin{subfigure}{0.32\linewidth}
        \centering
        \includegraphics[width=\linewidth]{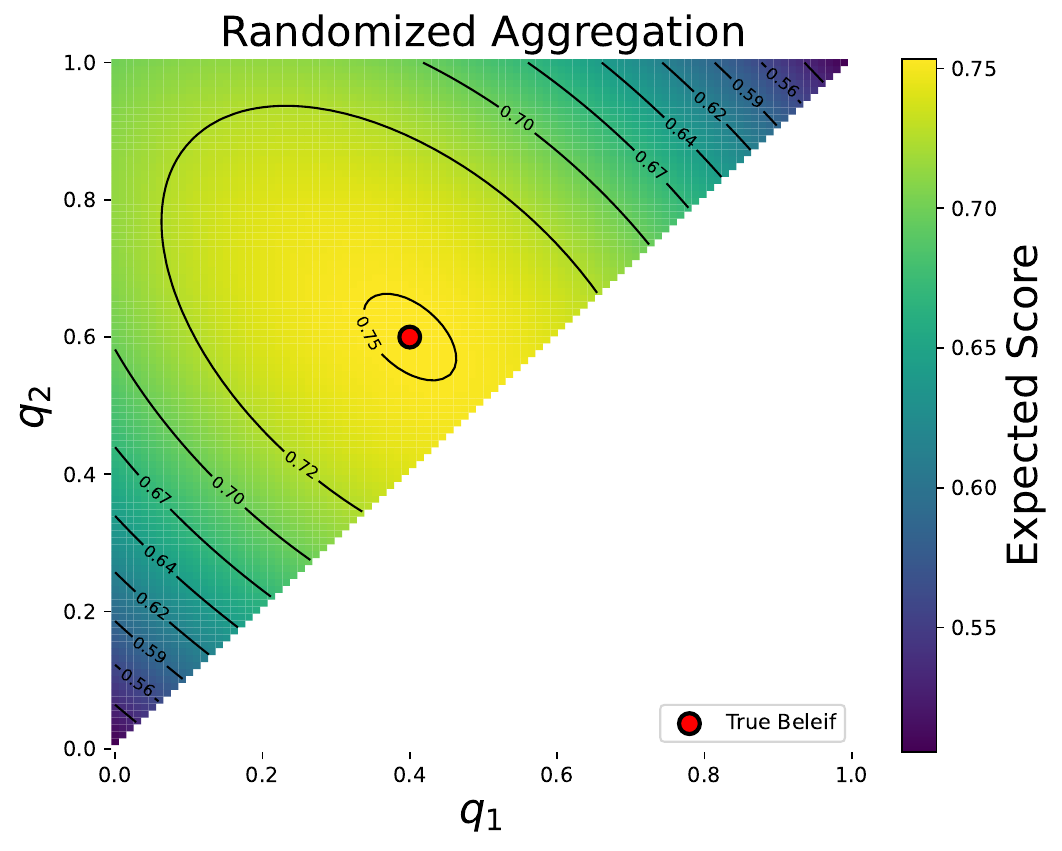}
        \caption{}
    \end{subfigure}
    \caption{(Left-to-right) In the first figure we simulate the scoring rule where $\rho$ is a dictatorship with a fixed mixing weight of 0.5, in the middle figure we simulate the scoring rule with min-max $\rho$ (pessimistic decision maker) and in the last figure we simulate the scoring rule where the aggregation is a randomized dictatorship and the forecaster obtains a distribution $\theta=\mathcal{U}[0,1]$ over $\rho$. The lower half of the figure is not plotted since that corresponds to region $q_1>q_2$, i.e. lower probability being greater than upper probability}
    \label{fig:simulations}
\end{figure}
\Cref{fig:simulations} highlights that the randomized scoring rule $s_\theta$ is strictly proper for imprecise forecasts as it has the highest expected score for the forecaster only when the forecaster reports his true belief. While in other cases of using a deterministic imprecise scoring rule $s_\rho$, if DM provides a $\rho$ such that it is a dictatorship, such as in the case of ~\Cref{fig:simulations}(a), the scoring rule is proper; however, the forecaster can lie by reporting the dictator. This can be inferred from the contour that the point $[0.5,0.5]$, which corresponds to the precise forecast $0.5$, also has the highest expected score. With $\rho$ being a min-max rule, the scoring rule $s_\rho$  is proper but not strictly as other imprecise forecasts allow the forecaster to obtain the same expected score. For our implementation we consider $\cA=[0,1]$ and $u(a,o):=(o-a)^2$ to satisfy ~\Cref{lemma:strictness-for-precise-distributions}. 

\end{bibunit}

\end{document}